\documentclass{article} 
\usepackage{iclr2027_conference,times}

\usepackage{amsmath,amsfonts,bm}

\def\eqref#1{equation~\ref{#1}}

\def\1{\bm{1}}

\def\vtheta{{\bm{\theta}}}

\def\vc{{\bm{c}}}

\def\vx{{\bm{x}}}

\DeclareMathAlphabet{\mathsfit}{\encodingdefault}{\sfdefault}{m}{sl}
\SetMathAlphabet{\mathsfit}{bold}{\encodingdefault}{\sfdefault}{bx}{n}

\usepackage{amsmath,amssymb}
\providecommand{\vc}{\mathbf{c}}
\providecommand{\vx}{\mathbf{x}}

\providecommand{\vtheta}{\boldsymbol{\theta}}
\providecommand{\vphi}{\boldsymbol{\phi}}

\usepackage[T1]{fontenc}
\usepackage{amsmath,amssymb}
\usepackage{booktabs,tabularx,array}
\usepackage{xcolor}
\usepackage{listings}
\usepackage{tcolorbox}
\tcbuselibrary{skins,breakable,listings}

\definecolor{PromptSystem}{HTML}{552583}
\definecolor{PromptReason}{HTML}{A9A6B1}
\definecolor{PromptFast}{HTML}{FDB927}
\definecolor{PromptReference}{HTML}{552583}
\definecolor{PromptInk}{HTML}{241635}

\newtcblisting{agentprompt}[2]{%
  enhanced,
  breakable,
  listing only,
  listing engine=listings,
  colback=#1!3!white,
  colframe=#1,
  colbacktitle=#1!12!white,
  coltitle=PromptInk,
  coltext=PromptInk,
  title={#2},
  title after break={#2 (continued)},
  fonttitle=\sffamily\bfseries\small,
  boxrule=0.45pt,
  leftrule=1.6pt,
  arc=1mm,
  left=2.2mm,
  right=2.2mm,
  top=1.5mm,
  bottom=1.5mm,
  before skip=6pt,
  after skip=6pt,
  listing options={%
    language={},
    basicstyle=\ttfamily\footnotesize\color{PromptInk},
    columns=fullflexible,
    keepspaces=true,
    showstringspaces=false,
    showspaces=false,
    showtabs=false,
    breaklines=true,
    breakatwhitespace=false,
    breakindent=0pt,
    tabsize=4,
    aboveskip=0pt,
    belowskip=0pt
  }
}

\usepackage{amsmath}   
\usepackage{amsthm}
\newtheorem{proposition}{Proposition}
\usepackage{booktabs}  
\usepackage{multirow}  
\usepackage{graphicx}  
\usepackage{colortbl}
\usepackage{hyperref}
\usepackage{url}
\usepackage{subcaption}

\title{Act First, Reason Later: Accelerating On-Policy Distillation for Multi-Turn Agents via Reference-Conditioned Inverse Dynamics}

\author{%
Zubin Zheng\textsuperscript{1}\thanks{Equal contribution. This work was partially conducted during Zubin’s internship at DeepCybo.}, 
Jiahao Wu\textsuperscript{3}\footnotemark[1], 
Shaofeng Zhang\textsuperscript{1,2}, 
Zhirui Zhang\textsuperscript{4}, 
Yew-Soon Ong\textsuperscript{5}, 
Shengcai Liu\textsuperscript{1}\thanks{Corresponding author. Email: \texttt{liusc3@sustech.edu.cn}.}
\\
\textsuperscript{1}Guangdong Provincial Key Laboratory of Brain-Inspired Intelligent Computation,\\ 
Department of Computer Science and Engineering, Southern University of Science and Technology\\
\textsuperscript{2}Zhongguancun Academy\\
\textsuperscript{3}Hong Kong Polytechnic University\\
\textsuperscript{4}DeepCybo\\
\textsuperscript{5}College of Computing \& Data Science, Nanyang Technological University
}

\iclrfinalcopy 
\begin{document}

\maketitle
\lhead{Preprint}

\begin{abstract}
On-policy distillation (OPD) trains multi-turn language agents with dense teacher supervision on student-generated responses. 
However, standard think-then-act rollouts require lengthy reasoning before each short action, delaying environment transitions and experience collection. 
Generating actions directly reduces this delay but can degrade rollout quality. 
To address this, we propose \mbox{\textbf{ActFirst-OPD}}, an~act-first, reason-later training framework that decouples environment interaction from full-response generation. 
The student infers and executes actions through reference-conditioned inverse dynamics using its current interaction context and a reference next observation, and switches to autonomous next-action prediction when the resulting transition deviates from the reference trajectory. 
From the collected interaction contexts, the student asynchronously generates full think-then-act responses for token-level teacher
supervision. 
Experiments across 0.6B-, 1.7B-, and 4B-parameter Qwen3 students show that ActFirst-OPD achieves average wall-clock training speedups of $2.3\times$ on ALFWorld, $1.8\times$ on WebShop, and $4.9\times$ on ScienceWorld over Vanilla OPD.
It matches or exceeds all compared OPD baselines in mean task success rate across eight of nine benchmark--model settings.
These results demonstrate that reasoning need not block acting during multi-turn agent distillation.
\end{abstract}

\section{Introduction}
\label{sec:introduction} 
Large language models (LLMs) are evolving from static text generators into agents that reason and act in interactive environments~\citep{react}. 
In multi-turn tasks, agents execute actions based on current observations, driving environment transitions that produce the next observation for subsequent reasoning and action~\citep{alfworld,webshop,scienceworld}. 
On-policy distillation (OPD)~\citep{minillm,gkd} trains such agents with dense token-level teacher supervision on student-generated responses~\citep{tcod,turnopd}.
Unlike supervised fine-tuning (SFT) on offline trajectories, OPD continually collects turn-level responses under the evolving student policy, covering response prefixes the student may encounter at inference and mitigating exposure bias~\citep{minillm,gkd}.

\begin{figure}[t]
    \centering
    \includegraphics[width=\linewidth]{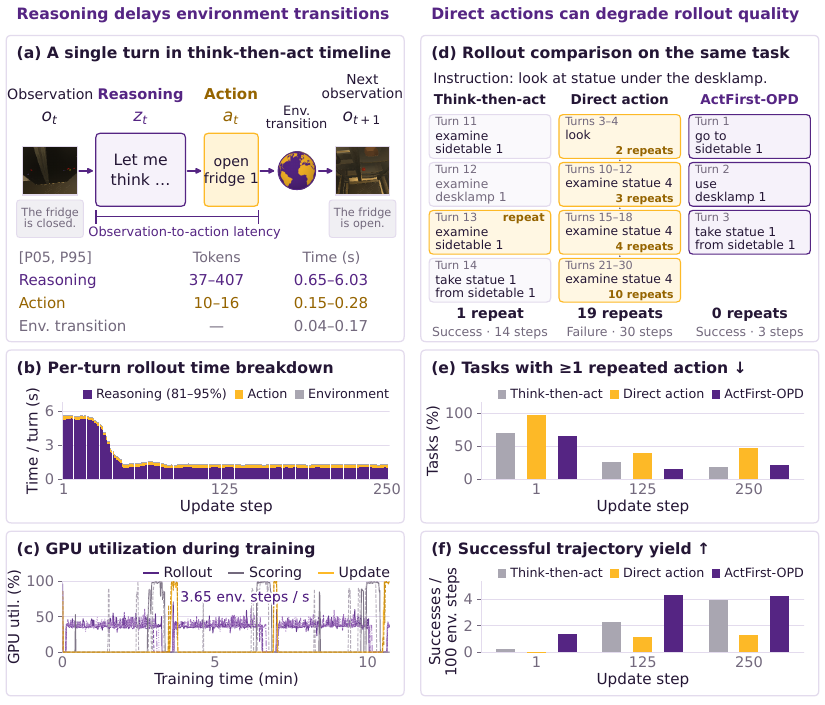}
    \caption{\textbf{Rollout latency and quality on ALFWorld with Qwen3-1.7B.}
    Ranges in (a) denote the 5th--95th percentiles.
    (c) shows GPU utilization during the first 10 minutes of training.
    In (d--f), all rollout strategies use the same frozen Vanilla OPD
    checkpoint at each update step, with the same 3,553 training tasks, and a 30-turn limit.
    (d) shows selected turns from the same task at update 1.}
    \label{fig:motivation}
\end{figure}

However, online experience collection remains costly in multi-turn OPD~\citep{turnopd,reopd}.
In standard think-then-act rollouts, the student completes lengthy reasoning before submitting a short action at each turn, delaying the environment transition (Figure~\ref{fig:motivation}(a)).
Our profiling of Qwen3-1.7B on ALFWorld shows that reasoning accounts for
81--95\% of per-turn rollout time (Figure~\ref{fig:motivation}(b)).
These delays accumulate across turns, ultimately increasing training time (Figure~\ref{fig:motivation}(c)).
We characterize this as a reasoning-blocked transition bottleneck in experience collection.
Generating actions directly while deferring full-response generation can reduce this delay but may degrade rollout quality.
Controlled comparisons show that direct-action rollouts exhibit more unproductive repetition (Figure~\ref{fig:motivation}(d, e))
and yield fewer successful trajectories per 100 environment transitions
(Figure~\ref{fig:motivation}(f)).
These observations raise the question: 
\emph{Can the student act to advance the environment without waiting for full reasoning while maintaining rollout quality?}

To address this challenge, we propose \textbf{ActFirst-OPD}, an act-first, reason-later training framework that decouples environment interaction from full-response generation.
Our key idea is to use high-quality reference trajectories to provide local transition targets for fast action generation, maintaining rollout quality.
Drawing on inverse dynamics~\citep{ridm}, we have the student infer and execute
an action from its current interaction context and a reference next observation.
When interaction deviates from the reference trajectory, the student
switches to autonomous next-action prediction for the rest of the rollout.
While fast actions advance the environment, the same student asynchronously generates full think-then-act responses from the collected interaction contexts without using reference next observations for token-level teacher supervision.
This preserves full-response distillation while removing reasoning from the critical path of environment transitions.

The main contributions of this work are summarized below.
\begin{itemize}
    \item We identify and characterize reasoning-blocked environment transitions in multi-turn OPD and the rollout-quality risks of direct action through profiling and controlled comparisons.

    \item We introduce ActFirst-OPD, combining reference-conditioned inverse dynamics and next-action prediction with asynchronous full-response generation for distillation. We also derive the idealized rollout speedup and its upper bound, and measure empirical gains.

     \item Across 0.6B-, 1.7B-, and 4B-parameter Qwen3 students, ActFirst-OPD achieves average wall-clock training speedups of $2.3\times$ on ALFWorld, $1.8\times$ on WebShop, and $4.9\times$ on ScienceWorld over Vanilla OPD under matched hardware and update steps.
     It matches or exceeds all compared OPD baselines in mean task success rate across eight of nine benchmark--model settings.
\end{itemize}

\section{Related Work}
\label{sec:related_work} 
\paragraph{On-Policy Distillation.}
Early OPD methods train students with dense teacher supervision on student-generated sequences~\citep{minillm,gkd}.
Applications span reasoning, strong-to-weak distillation, and integration of domain-specific capabilities~\citep{thinkingmachineslab-opd,qwen3,deepseekv4}.
Recent studies examine OPD training dynamics and mechanisms~\citep{opd-survey,rethinking-opd}.
For multi-turn agents, TCOD~\citep{tcod} uses temporal curricula to mitigate trajectory-level KL instability, while TurnOPD~\citep{turnopd} combines adaptive rollout depth with turn-level loss weighting.
Our work addresses delays from pre-action reasoning and the rollout-quality risks of direct action generation.
Additional comparisons with recent OPD methods are provided in Appendix~\ref{app:extended-related-work}.

\paragraph{Multi-Turn LLM Agents.}
LLM agents interleave reasoning and actions, as exemplified by ReAct~\citep{react}, for tasks such as embodied planning~\citep{alfworld} and web navigation~\citep{webshop}.
Recent coding agents, including Claude Code~\citep{claudecode} and Codex~\citep{codex}, and agent harnesses~\citep{agentharness,codeasharness} support long-horizon, tool-using workflows.
Despite these advances, training such agents remains challenging due to credit assignment under sparse or delayed rewards~\citep{gigpo} and compounding distribution shift across turns~\citep{simpletir}.

\paragraph{Inverse Dynamics.}
Inverse dynamics models infer actions from state transitions and have been used for self-supervised reinforcement learning~\citep{pathak2017curiosity} and imitation without expert action annotations~\citep{bco}.
Recent embodied-agent methods use predicted visual futures or their representations to guide action generation~\citep{vpp,vera}.
Most closely related to our formulation, RIDM~\citep{ridm} infers actions from the learner's current observation and the expert's next observation to drive environment interaction without expert actions.
We adapt this conditioning scheme to multi-turn OPD,
using reference next observations to guide fast student actions
and maintain rollout quality when collecting interaction contexts.

\section{Preliminaries}
\label{sec:preliminaries}

\paragraph{Multi-Turn Agent Interaction.}
We consider a language agent interacting with an environment to complete
a task specified by an instruction $q$.
At turn $t$, the agent receives an observation $o_t$ and maintains
the history of preceding interactions,
\begin{equation}
    h_t=(o_1,a_1,\ldots,o_{t-1},a_{t-1}),
    \qquad h_1=\varnothing,
    \label{eq:agent-history}
\end{equation}
where $a_i$ is the action submitted to the environment at turn $i$;
past reasoning is excluded from $h_t$.
We write the agent input schematically as interaction context
$\vc_t=q\oplus h_t\oplus o_t$, where $\oplus$ denotes concatenation
and benchmark-specific system instructions and formatting are implicit.
In standard think-then-act interaction~\citep{react}, the student policy
$\pi_{\vtheta}$ generates a full response
\begin{equation}
    y_t=(z_t,a_t)\sim\pi_{\vtheta}(\cdot\mid\vc_t),
    \label{eq:agent-response}
\end{equation}
where $z_t$ denotes the reasoning segment and $a_t$ denotes the executable
action segment; only $a_t$ is submitted to the environment.
The environment executes $a_t$ and returns the next observation $o_{t+1}$.
Interaction ends when the environment terminates or the prescribed turn limit is reached,
yielding the trajectory
$\tau=(o_1,a_1,o_2,\ldots,o_T,a_T,o_{T+1})$
over $T$ interaction turns.
Prompt construction, history management, and thinking budgets are
detailed in Appendices~\ref{app:prompts},
\ref{app:memory-management}, and~\ref{app:thinking}, respectively.

\paragraph{On-Policy Distillation for Multi-Turn Agents.}
Multi-turn OPD trains a student policy $\pi_{\vtheta}$ with a frozen
teacher $\pi_{\vphi}$ on responses collected from online student
rollouts.
Let $\vx_t=(x_{t,1},\ldots,x_{t,L_t})$ denote the tokenization of the
full response $y_t=(z_t,a_t)$, where $L_t$ is the number of response
tokens at turn $t$.
For each response position $j\in\{1,\ldots,L_t\}$, define the token context
\begin{equation}
    \vc_{t,j}
    =
    \vc_t\oplus\vx_{t,<j},
    \label{eq:token-context}
\end{equation}
where
$\vx_{t,<j}=(x_{t,1},\ldots,x_{t,j-1})$
is the response prefix preceding token $x_{t,j}$.
Following prior studies~\citep{tcod}, multi-turn OPD minimizes the expected student-to-teacher reverse
Kullback--Leibler (KL) divergence over supervised token contexts,
\begin{equation}
    \mathcal{L}_{\mathrm{OPD}}(\vtheta)
    =
    \mathbb{E}
    \left[
        D_{\mathrm{KL}}
        \left(
            \pi_{\vtheta}(\cdot\mid\vc_{t,j})
            \,\|\,
            \pi_{\vphi}(\cdot\mid\vc_{t,j})
        \right)
    \right].
    \label{eq:opd-objective}
\end{equation}
The expectation is over collected turn-level full responses $y_t$ at turns $t$ of interaction trajectories $\tau$, and supervised token positions $j$ within each response.
For each token context, the reverse KL is
\begin{equation}
    D_{\mathrm{KL}}
    \left(
        \pi_{\vtheta}(\cdot\mid\vc_{t,j})
        \,\|\,
        \pi_{\vphi}(\cdot\mid\vc_{t,j})
    \right)
    =
    \sum_{v\in\mathcal{V}}
    \pi_{\vtheta}(v\mid\vc_{t,j})
    \log
    \frac{
        \pi_{\vtheta}(v\mid\vc_{t,j})
    }{
        \pi_{\vphi}(v\mid\vc_{t,j})
    },
    \label{eq:reverse-kl}
\end{equation}
where $\mathcal{V}$ denotes the token vocabulary.
In practice, we optimize a sampled policy surrogate of
Eq.~\ref{eq:opd-objective}; the update rule, clipping, masking,
and loss normalization are detailed in
Appendix~\ref{app:opd-optimization}.

\section{ActFirst-OPD}
\label{sec:method}

ActFirst-OPD decouples environment interaction from full-response generation, as illustrated in
Figure~\ref{fig:actfirst-framework}.
The student advances the environment through actions conditioned
on reference next observations, with autonomous next-action prediction
once the rollout diverges from the reference trajectory.
Meanwhile, full responses are generated asynchronously
from collected contexts for distillation, allowing subsequent
environment interactions to proceed without waiting for reasoning.

\begin{figure}[t]
    \centering
    \includegraphics[width=\linewidth]{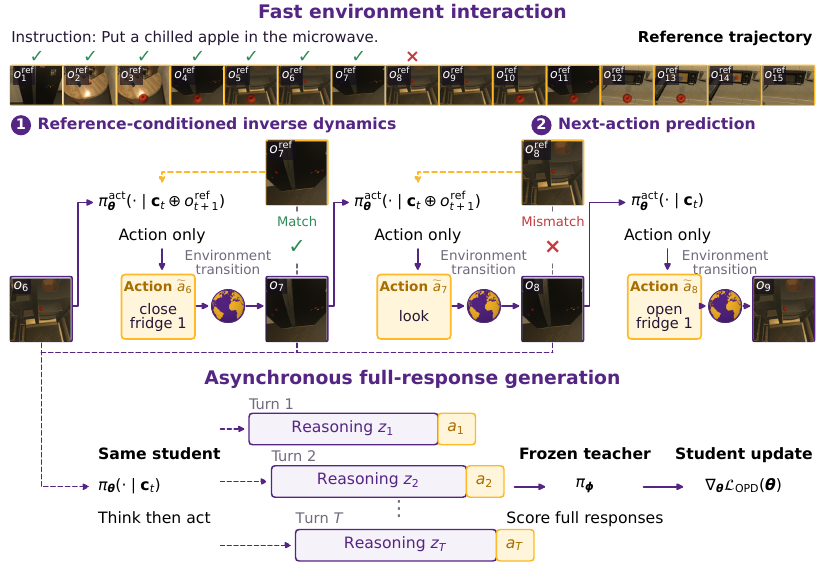}
    \caption{\textbf{Overview of ActFirst-OPD.}
    Top: Reference-conditioned inverse dynamics generates fast actions
    $\tilde a_t$ for environment interaction, with autonomous next-action
    prediction once the rollout diverges from the reference trajectory.
    Bottom: The same student asynchronously generates full responses
    $y_t=(z_t,a_t)$ from collected interaction contexts for teacher
    supervision.}
    \label{fig:actfirst-framework}
\end{figure}

\subsection{Reference-Conditioned Inverse-Dynamics Rollout}
\label{sec:reference-conditioned-rollout}

Given a high-quality offline reference trajectory for training task $q$, let
$o_1^{\mathrm{ref}},o_2^{\mathrm{ref}},\ldots$
denote its observations.
While the rollout remains aligned with the reference trajectory,
the reference next observation $o_{t+1}^{\mathrm{ref}}$ serves
as a local transition target.
Following the inverse-dynamics formulation~\citep{ridm},
the student infers an action from its actual interaction context
$\vc_t$ and this target:
\begin{equation}
    \tilde a_t
    \sim
    \pi_{\vtheta}^{\mathrm{act}}
    \left(\cdot\mid\vc_t\oplus o_{t+1}^{\mathrm{ref}}\right),
    \label{eq:reference-conditioned-action}
\end{equation}
where $\pi_{\vtheta}^{\mathrm{act}}$ denotes the same student
prompted to generate only an action, without preceding reasoning.
During interaction, $h_t$ records actual observations and
executed fast actions $\tilde a_1,\ldots,\tilde a_{t-1}$.
Thus, $\vc_t=q\oplus h_t\oplus o_t$ remains the student's actual
interaction context.
Executing $\tilde a_t$ produces $o_{t+1}$ and advances the rollout.

\subsection{Autonomous Next-Action Prediction Fallback}
\label{sec:verification-fallback}

After executing a reference-conditioned action $\tilde a_t$,
we apply benchmark-specific transition checks (Appendix~\ref{app:transition-checks}) to determine whether
the resulting observation $o_{t+1}$ aligns with the
target $o_{t+1}^{\mathrm{ref}}$.
Verification concerns the transition outcome and does not require
the generated action to match the reference action.
If verification succeeds, reference guidance continues at the next turn.

Once verification fails, reference-conditioned inverse dynamics is disabled for the
remainder of the rollout.
At each subsequent turn, the student performs autonomous next-action prediction (NAP) using its interaction context alone:
\begin{equation}
    \tilde a_t
    \sim
    \pi_{\vtheta}^{\mathrm{act}}(\cdot\mid\vc_t).
    \label{eq:autonomous-action}
\end{equation}
Interaction continues from the actual state reached by the student,
preserving action-only generation while avoiding continued reliance
on the inapplicable reference trajectory.

\subsection{Asynchronous Full-Response Generation}
\label{sec:asynchronous-full-response}

For each interaction context $\vc_t$ collected during fast rollouts,
the student generates a full think-then-act response
$y_t=(z_t,a_t)\sim\pi_{\vtheta}(\cdot\mid\vc_t)$ using the standard
agent input, without reference next observations.
Full-response generation proceeds asynchronously, allowing
requests from different turns to overlap with one another
and with subsequent environment interactions.
The action $a_t$ need not match the executed fast action $\tilde a_t$
and is not submitted to the environment.

The training batch consists of student-sampled full responses
at contexts visited by the fast rollout.
The frozen teacher $\pi_{\vphi}$ scores each full response under
the same token contexts $\vc_{t,j}$ defined in
Section~\ref{sec:preliminaries}.
These turn-level samples are used to optimize the OPD objective
in Eq.~\ref{eq:opd-objective}.
This preserves full-response supervision while removing full-response
generation from the dependency chain of environment transitions.
At evaluation time, the student uses standard think-then-act
interaction without reference next observations.
We discuss the context distribution induced by fast rollouts and its implications
for on-policy distillation in
Appendix~\ref{app:additional-discussion}.

\paragraph{Idealized Rollout Efficiency.}
We analyze rollout completion time, measured until both environment
interaction and generation of all corresponding full responses
have finished.
Consider a fixed $T$-turn rollout with constant generation latencies
$\ell_{\mathrm{fast}}$ and $\ell_{\mathrm{full}}$ for fast actions and
full responses, respectively, including prefill and autoregressive
decoding, where
$0<\ell_{\mathrm{fast}}\le\ell_{\mathrm{full}}$.
In the idealized model, both requests start as soon as $\vc_t$
becomes available.
We neglect environment and other non-generation overheads and assume
sufficient concurrency so that overlapping requests do not increase
their generation latencies.

\begin{proposition}[Idealized rollout speedup]
\label{prop:rollout-speedup}
Under the above assumptions, standard think-then-act rollout takes
$T\ell_{\mathrm{full}}$, whereas ActFirst-OPD completes in
$(T-1)\ell_{\mathrm{fast}}+\ell_{\mathrm{full}}$.
The resulting idealized rollout speedup satisfies
\begin{equation}
    S_{\mathrm{ideal}}(T)
    =
    \frac{T\ell_{\mathrm{full}}}
    {(T-1)\ell_{\mathrm{fast}}+\ell_{\mathrm{full}}}
    \le
    \min\left\{
        T,\,
        \frac{\ell_{\mathrm{full}}}{\ell_{\mathrm{fast}}}
    \right\}.
    \label{eq:ideal-rollout-speedup}
\end{equation}
\end{proposition}

The proof of Proposition~\ref{prop:rollout-speedup} is provided in Appendix~\ref{app:proof-rollout-speedup}.
Longer-horizon rollouts amortize the final full-response generation cost,
with $S_{\mathrm{ideal}}(T)$ approaching
$\ell_{\mathrm{full}}/\ell_{\mathrm{fast}}$ as $T$ increases (Eq.~\ref{eq:ideal-rollout-speedup-limit}).
ActFirst-OPD additionally generates a fast action at each turn,
so realized speedup requires sufficient serving concurrency
for asynchronous overlap to offset this extra workload
(Appendix~\ref{app:finite-capacity},
Eq.~\ref{eq:concurrency-speedup-condition}).

\section{Experiments}
\label{sec:experiments}

\subsection{Experimental Setup}

\paragraph{Tasks, Models, and Reference Trajectories.}
We evaluate on ALFWorld~\citep{alfworld} for embodied planning,
WebShop~\citep{webshop} for web navigation, and
ScienceWorld~\citep{scienceworld} for scientific reasoning,
with respective interaction limits of 30, 15, and 30 turns.
Following prior multi-turn OPD studies~\citep{tcod,turnopd}, we train
Qwen3-0.6B, 1.7B, and 4B students~\citep{qwen3} using a task-specialized
Qwen3-8B teacher trained with GiGPO~\citep{gigpo} for each benchmark.
ActFirst-OPD uses offline reference trajectories selected from
public demonstrations, model-generated rollouts, and oracle solutions.
Data splits are detailed in Appendix~\ref{app:benchmarks};
teacher preparation and reference construction
are described in Appendix~\ref{app:teacher-reference}.

\paragraph{Baselines and Implementation.}
We compare against Vanilla OPD, TCOD-F2B~\citep{tcod}, and
TurnOPD~\citep{turnopd}, with results from zero-shot students
and task-specialized teachers provided for reference.
We also include Ours w/o ID, which replaces reference-conditioned
inverse dynamics (ID) with direct action generation without references while retaining asynchronous full-response generation.
Within each benchmark--model setting, methods share student
initialization, the teacher, and environment settings, while
retaining their respective rollout and optimization rules.
The main comparison in Section~\ref{sec:main-results} uses 250 training updates, rollout batches
of 16 tasks, and optimization batches of 64 turn-level full responses.
All OPD methods are implemented using Trinity-RFT~\citep{trinityrft},
with vLLM~\citep{vllm} for inference and VERL~\citep{verl}
for optimization.
Each training run uses eight NVIDIA A100-SXM4-80GB GPUs:
four for student rollouts, two for teacher scoring, and two for student optimization.

\paragraph{Evaluation Protocols and Metrics.}
We evaluate on 140 seen and 134 unseen ALFWorld tasks, 100 held-out
WebShop tasks, and 150 held-out ScienceWorld tasks.
All models use standard think-then-act interaction without references
during evaluation.
We report success rate (SR), average interaction turns (Round), and training
wall-clock time; WebShop and ScienceWorld additionally report task
scores to measure partial completion.
Training wall-clock time excludes teacher training, reference construction, and evaluation.
Training speedup is Vanilla OPD's training wall-clock time divided
by the method's time for the same benchmark and student size.
Task-performance means and standard deviations are computed across
three evaluation seeds (42, 43, and 44) for the same trained checkpoint.

\begin{table}[t]
\centering
\caption{\textbf{Task performance and training cost on ALFWorld.}
Overall SR is reported as mean $\pm$ standard deviation across
three evaluation seeds.
Training time is in hours, with speedup relative to Vanilla OPD
for the same student size.
Among OPD methods, \textbf{bold} and \underline{underlined} values
mark the best and second-best results per metric, respectively;
shading highlights the shortest training time.}
\label{tab:main-alfworld}
\small
\setlength{\tabcolsep}{3pt}
\renewcommand{\arraystretch}{1.15}
\resizebox{\linewidth}{!}{%
\begin{tabular}{ccl*{13}{c}}
\toprule
\multicolumn{2}{c}{Model} & \multicolumn{1}{c}{\multirow{3}{*}{Method}} & \multicolumn{13}{c}{ALFWorld} \\
\cmidrule(lr){1-2}\cmidrule(lr){4-16}
\multirow{2}{*}{Teacher} & \multirow{2}{*}{Student} & & \multirow{2}{*}{Pick} & \multirow{2}{*}{Look} & \multirow{2}{*}{Clean} & \multirow{2}{*}{Heat} & \multirow{2}{*}{Cool} & \multirow{2}{*}{Pick2} & \multicolumn{2}{c}{Seen} & \multicolumn{2}{c}{Unseen} & \multicolumn{2}{c}{All} & \multirow{2}{*}{\shortstack{Training Wall-clock\\Time (Speedup)}} \\
\cmidrule(lr){10-11}\cmidrule(lr){12-13}\cmidrule(lr){14-15}
& & & & & & & & & SR$\uparrow$ & Round$\downarrow$ & SR$\uparrow$ & Round$\downarrow$ & SR$\uparrow$ & Round$\downarrow$ & \\
\midrule
\rowcolor{gray!12}\multicolumn{16}{c}{\textit{Qwen3 Series}} \\
\midrule
- & 0.6B & \multirow{3}{*}{Zero-Shot} & 0.00 & 6.45 & 0.00 & 0.00 & 0.00 & 0.00 & 0.00 & 30.00 & 1.49 & 29.69 & 0.73 & 29.85 & N/A \\
- & 1.7B &  & 40.68 & 16.13 & 13.79 & 12.82 & 8.70 & 4.88 & 17.86 & 27.68 & 17.16 & 27.28 & 17.52 & 27.48 & N/A \\
- & 4B &  & 89.83 & 77.42 & 60.34 & 58.97 & 28.26 & 63.41 & 55.71 & 21.92 & 71.64 & 20.10 & 63.50 & 21.03 & N/A \\
\midrule
8B & - & GiGPO & 98.31 & 90.32 & 94.83 & 79.49 & 86.96 & 73.17 & 89.29 & 10.04 & 87.31 & 11.18 & 88.32 & 10.59 & N/A \\
\midrule
\multirow{15}{*}{\shortstack{8B-\\GiGPO}} & \multirow{5}{*}{0.6B} & Vanilla OPD & $80.79$ & $\underline{55.91}$ & $66.67$ & $\underline{59.83}$ & $52.90$ & $\underline{34.96}$ & $61.67$ & $17.52$ & $59.20$ & $17.90$ & $60.46\pm 3.15$ & $17.70$ & $4.18\,\mathrm{h}\,(\times 1.00)$ \\
 &  & TCOD-F2B & $\underline{85.88}$ & $53.76$ & $\underline{67.24}$ & $55.56$ & $\underline{60.14}$ & $32.52$ & $\underline{63.81}$ & $\underline{16.22}$ & $\underline{59.45}$ & $\underline{17.41}$ & $\underline{61.68\pm 1.59}$ & $\underline{16.80}$ & $\underline{2.15\,\mathrm{h}\,(\times 1.94)}$ \\
 &  & TurnOPD & $72.88$ & $\boldsymbol{60.22}$ & $48.85$ & $39.32$ & $53.62$ & $21.95$ & $51.43$ & $19.18$ & $50.00$ & $19.72$ & $50.73\pm 2.89$ & $19.44$ & $2.87\,\mathrm{h}\,(\times 1.46)$ \\
 &  & Ours w/o ID & $72.88$ & $\underline{55.91}$ & $51.15$ & $26.50$ & $34.78$ & $13.01$ & $43.81$ & $20.89$ & $45.02$ & $21.13$ & $44.40\pm 6.38$ & $21.01$ & $2.30\,\mathrm{h}\,(\times 1.82)$ \\
 &  & \textbf{ActFirst-OPD (Ours)} & $\boldsymbol{87.01}$ & $\boldsymbol{60.22}$ & $\boldsymbol{75.86}$ & $\boldsymbol{65.81}$ & $\boldsymbol{66.67}$ & $\boldsymbol{43.09}$ & $\boldsymbol{71.67}$ & $\boldsymbol{14.90}$ & $\boldsymbol{65.42}$ & $\boldsymbol{16.51}$ & $\boldsymbol{68.61\pm 0.89}$ & $\boldsymbol{15.69}$ & \cellcolor{yellow!25}$\boldsymbol{1.93\,\mathrm{h}\,(\times 2.17)}$ \\
\cmidrule(lr){2-16}
 & \multirow{5}{*}{1.7B} & Vanilla OPD & $\underline{92.66}$ & $\underline{63.44}$ & $74.14$ & $65.81$ & $56.52$ & $39.84$ & $70.48$ & $14.70$ & $64.68$ & $16.79$ & $67.64\pm 1.52$ & $15.72$ & $4.46\,\mathrm{h}\,(\times 1.00)$ \\
 &  & TCOD-F2B & $\underline{92.66}$ & $50.54$ & $\underline{81.61}$ & $64.96$ & $55.80$ & $\boldsymbol{57.72}$ & $73.33$ & $14.33$ & $66.92$ & $15.74$ & $70.19\pm 1.38$ & $15.02$ & $\underline{2.32\,\mathrm{h}\,(\times 1.92)}$ \\
 &  & TurnOPD & $90.96$ & $\underline{63.44}$ & $79.31$ & $\boldsymbol{67.52}$ & $\underline{64.49}$ & $\underline{54.47}$ & $\underline{74.76}$ & $\underline{14.01}$ & $\underline{69.40}$ & $\underline{15.30}$ & $\underline{72.14\pm 3.11}$ & $\underline{14.64}$ & $3.21\,\mathrm{h}\,(\times 1.39)$ \\
 &  & Ours w/o ID & $74.01$ & $59.14$ & $62.07$ & $52.14$ & $55.07$ & $30.08$ & $57.86$ & $18.10$ & $55.97$ & $18.97$ & $56.93\pm 3.85$ & $18.53$ & $2.48\,\mathrm{h}\,(\times 1.80)$ \\
 &  & \textbf{ActFirst-OPD (Ours)} & $\boldsymbol{93.79}$ & $\boldsymbol{82.80}$ & $\boldsymbol{85.06}$ & $\underline{66.67}$ & $\boldsymbol{73.19}$ & $\underline{54.47}$ & $\boldsymbol{80.95}$ & $\boldsymbol{12.62}$ & $\boldsymbol{73.88}$ & $\boldsymbol{14.27}$ & $\boldsymbol{77.49\pm 1.20}$ & $\boldsymbol{13.43}$ & \cellcolor{yellow!25}$\boldsymbol{1.95\,\mathrm{h}\,(\times 2.28)}$ \\
\cmidrule(lr){2-16}
 & \multirow{5}{*}{4B} & Vanilla OPD & $92.66$ & $80.65$ & $\boldsymbol{83.91}$ & $\underline{72.65}$ & $\underline{74.64}$ & $68.29$ & $84.05$ & $\underline{11.40}$ & $75.62$ & $13.55$ & $79.93\pm 0.73$ & $12.46$ & $4.56\,\mathrm{h}\,(\times 1.00)$ \\
 &  & TCOD-F2B & $92.66$ & $\underline{88.17}$ & $72.41$ & $58.12$ & $61.59$ & $73.17$ & $79.29$ & $12.65$ & $70.15$ & $15.11$ & $74.82\pm 1.10$ & $13.85$ & $2.72\,\mathrm{h}\,(\times 1.68)$ \\
 &  & TurnOPD & $\underline{93.79}$ & $83.87$ & $\underline{79.89}$ & $\boldsymbol{79.49}$ & $72.46$ & $\underline{76.42}$ & $\underline{84.52}$ & $11.58$ & $\underline{78.36}$ & $\underline{13.13}$ & $\underline{81.51\pm 1.47}$ & $\underline{12.34}$ & $3.86\,\mathrm{h}\,(\times 1.18)$ \\
 &  & Ours w/o ID & $\boldsymbol{95.48}$ & $84.95$ & $75.86$ & $65.81$ & $63.77$ & $60.98$ & $80.00$ & $12.28$ & $70.65$ & $14.75$ & $75.43\pm 0.92$ & $13.49$ & $\underline{2.59\,\mathrm{h}\,(\times 1.76)}$ \\
 &  & \textbf{ActFirst-OPD (Ours)} & $92.66$ & $\boldsymbol{90.32}$ & $78.74$ & $71.79$ & $\boldsymbol{81.88}$ & $\boldsymbol{78.86}$ & $\boldsymbol{86.19}$ & $\boldsymbol{11.01}$ & $\boldsymbol{78.86}$ & $\boldsymbol{12.82}$ & $\boldsymbol{82.60\pm 1.12}$ & $\boldsymbol{11.90}$ & \cellcolor{yellow!25}$\boldsymbol{1.93\,\mathrm{h}\,(\times 2.37)}$ \\
\bottomrule
\end{tabular}%
}
\end{table}

\subsection{Main Results}
\label{sec:main-results}
Across the nine benchmark--model settings, ActFirst-OPD completes
250 updates faster than Vanilla OPD, TCOD-F2B, and TurnOPD,
while matching or exceeding Vanilla OPD's mean SR in eight
(Tables~\ref{tab:main-alfworld}
and~\ref{tab:main-webshop-scienceworld}).
Averaging the per-size training speedups over Vanilla OPD gives
$2.3\times$, $1.8\times$, and $4.9\times$ on ALFWorld,
WebShop, and ScienceWorld, respectively.

ActFirst-OPD achieves the highest mean SR among the compared OPD methods
at all three student sizes on ALFWorld and ScienceWorld.
On ALFWorld, its gains over Vanilla OPD are 8.15, 9.85,
and 2.67 percentage points for 0.6B, 1.7B, and 4B, respectively.
These SR gains may partly stem from reference-conditioned inverse dynamics promoting
task progress during rollout collection and providing more useful
contexts for distillation, consistent with the ablation and
rollout-quality results
(Table~\ref{tab:ablation_average_across_scales};
Figure~\ref{fig:rollout-quality}).
On WebShop, ActFirst-OPD achieves a $1.80\times$ training speedup
over Vanilla OPD with only a 1.00-percentage-point decrease in mean SR
for Qwen3-1.7B.

While Tables~\ref{tab:main-alfworld}
and~\ref{tab:main-webshop-scienceworld} compare performance after
250 training updates, we also compare evaluation SR under comparable
training-time budgets using performance--time curves on ALFWorld
(Figure~\ref{fig:alfworld-performance-time}
in Appendix~\ref{app:additional-experimental-results}).
At approximately 1.5 hours of training, ActFirst-OPD achieves
the highest mean evaluation SR among the compared OPD methods
at all three student sizes, with particularly large
gains over Vanilla OPD and TurnOPD.

\begin{table}[t]
\centering
\caption{\textbf{Task performance and training cost on WebShop
and ScienceWorld.}
Score and SR are reported as mean $\pm$ standard deviation across
three evaluation seeds.
Training time is in hours, with speedup relative to Vanilla OPD
for the same benchmark and student size.
Among OPD methods, \textbf{bold} and \underline{underlined} values
mark the best and second-best Score, SR, and training time,
respectively; shading highlights the shortest training time.}
\label{tab:main-webshop-scienceworld}
\small
\setlength{\tabcolsep}{3pt}
\renewcommand{\arraystretch}{1.15}
\resizebox{\linewidth}{!}{%
\begin{tabular}{ccl*{8}{c}}
\toprule
\multicolumn{2}{c}{Model} & \multicolumn{1}{c}{\multirow{3}{*}{Method}} & \multicolumn{4}{c}{WebShop} & \multicolumn{4}{c}{ScienceWorld} \\
\cmidrule(lr){1-2}\cmidrule(lr){4-7}\cmidrule(lr){8-11}
\multirow{2}{*}{Teacher} & \multirow{2}{*}{Student} & & \multirow{2}{*}{Score$\uparrow$} & \multirow{2}{*}{SR$\uparrow$} & \multirow{2}{*}{Round} & \multirow{2}{*}{\shortstack{Training Wall-clock\\Time (Speedup)}} & \multirow{2}{*}{Score$\uparrow$} & \multirow{2}{*}{SR$\uparrow$} & \multirow{2}{*}{Round} & \multirow{2}{*}{\shortstack{Training Wall-clock\\Time (Speedup)}} \\
& & & & & & & & & & \\
\midrule
\rowcolor{gray!12}\multicolumn{11}{c}{\textit{Qwen3 Series}} \\
\midrule
- & 0.6B & \multirow{3}{*}{Zero-Shot} & 49.54 & 11.00 & 4.56 & N/A & 5.59 & 0.00 & 19.86 & N/A \\
- & 1.7B &  & 51.10 & 15.67 & 5.22 & N/A & 22.28 & 6.89 & 19.65 & N/A \\
- & 4B &  & 45.37 & 13.33 & 6.75 & N/A & 44.36 & 20.67 & 19.19 & N/A \\
\midrule
8B & - & GiGPO & 71.54 & 43.33 & 4.28 & N/A & 55.58 & 35.33 & 15.40 & N/A \\
\midrule
\multirow{15}{*}{\shortstack{8B-\\GiGPO}} & \multirow{5}{*}{0.6B} & Vanilla OPD & $\boldsymbol{69.29\pm 1.23}$ & $\underline{36.67\pm 2.08}$ & $4.32$ & $2.81\,\mathrm{h}\,(\times 1.00)$ & $35.35\pm 0.91$ & $10.00\pm 1.33$ & $20.45$ & $10.31\,\mathrm{h}\,(\times 1.00)$ \\
 &  & TCOD-F2B & $67.79\pm 1.57$ & $36.33\pm 1.15$ & $4.23$ & $2.42\,\mathrm{h}\,(\times 1.16)$ & $35.39\pm 1.02$ & $10.00\pm 1.33$ & $18.52$ & $6.30\,\mathrm{h}\,(\times 1.64)$ \\
 &  & TurnOPD & $67.61\pm 1.63$ & $36.33\pm 2.31$ & $4.31$ & $2.88\,\mathrm{h}\,(\times 0.97)$ & $35.45\pm 0.60$ & $10.22\pm 0.39$ & $19.58$ & $6.31\,\mathrm{h}\,(\times 1.63)$ \\
 &  & Ours w/o ID & $67.93\pm 1.11$ & $36.00\pm 1.00$ & $4.15$ & \cellcolor{yellow!25}$\boldsymbol{1.40\,\mathrm{h}\,(\times 2.01)}$ & $\boldsymbol{39.16\pm 0.27}$ & $\underline{11.11\pm 1.39}$ & $15.04$ & $\underline{2.15\,\mathrm{h}\,(\times 4.79)}$ \\
 &  & \textbf{ActFirst-OPD (Ours)} & $\underline{68.35\pm 0.44}$ & $\boldsymbol{38.00\pm 1.00}$ & $4.76$ & $\underline{1.46\,\mathrm{h}\,(\times 1.93)}$ & $\underline{37.06\pm 1.47}$ & $\boldsymbol{12.67\pm 1.77}$ & $18.15$ & \cellcolor{yellow!25}$\boldsymbol{2.01\,\mathrm{h}\,(\times 5.12)}$ \\
\cmidrule(lr){2-11}
 & \multirow{5}{*}{1.7B} & Vanilla OPD & $\underline{69.50\pm 0.43}$ & $\boldsymbol{40.33\pm 0.58}$ & $4.04$ & $3.10\,\mathrm{h}\,(\times 1.00)$ & $45.47\pm 1.71$ & $\underline{22.89\pm 2.14}$ & $15.24$ & $10.37\,\mathrm{h}\,(\times 1.00)$ \\
 &  & TCOD-F2B & $68.83\pm 0.59$ & $38.33\pm 1.53$ & $4.15$ & $2.67\,\mathrm{h}\,(\times 1.16)$ & $\underline{46.06\pm 0.26}$ & $22.45\pm 0.39$ & $15.77$ & $6.53\,\mathrm{h}\,(\times 1.59)$ \\
 &  & TurnOPD & $\underline{69.50\pm 0.64}$ & $\underline{39.67\pm 1.15}$ & $3.99$ & $2.77\,\mathrm{h}\,(\times 1.12)$ & $43.61\pm 1.62$ & $20.44\pm 0.77$ & $15.50$ & $6.37\,\mathrm{h}\,(\times 1.63)$ \\
 &  & Ours w/o ID & $68.44\pm 1.93$ & $38.67\pm 2.08$ & $4.06$ & \cellcolor{yellow!25}$\boldsymbol{1.70\,\mathrm{h}\,(\times 1.82)}$ & $43.38\pm 1.61$ & $20.00\pm 2.41$ & $15.05$ & $\underline{2.56\,\mathrm{h}\,(\times 4.06)}$ \\
 &  & \textbf{ActFirst-OPD (Ours)} & $\boldsymbol{69.76\pm 1.11}$ & $39.33\pm 1.53$ & $3.99$ & $\underline{1.72\,\mathrm{h}\,(\times 1.80)}$ & $\boldsymbol{46.57\pm 2.38}$ & $\boldsymbol{24.00\pm 1.76}$ & $15.50$ & \cellcolor{yellow!25}$\boldsymbol{2.16\,\mathrm{h}\,(\times 4.80)}$ \\
\cmidrule(lr){2-11}
 & \multirow{5}{*}{4B} & Vanilla OPD & $\boldsymbol{69.90\pm 1.12}$ & $\boldsymbol{42.00\pm 1.00}$ & $4.23$ & $3.96\,\mathrm{h}\,(\times 1.00)$ & $50.74\pm 0.13$ & $29.34\pm 1.15$ & $16.49$ & $12.40\,\mathrm{h}\,(\times 1.00)$ \\
 &  & TCOD-F2B & $69.25\pm 1.15$ & $\boldsymbol{42.00\pm 1.00}$ & $4.12$ & $3.51\,\mathrm{h}\,(\times 1.13)$ & $\underline{53.47\pm 2.14}$ & $\underline{30.22\pm 1.39}$ & $17.25$ & $8.47\,\mathrm{h}\,(\times 1.46)$ \\
 &  & TurnOPD & $\underline{69.64\pm 1.76}$ & $41.33\pm 2.52$ & $4.17$ & $3.52\,\mathrm{h}\,(\times 1.13)$ & $51.64\pm 1.53$ & $30.00\pm 3.06$ & $16.49$ & $7.93\,\mathrm{h}\,(\times 1.56)$ \\
 &  & Ours w/o ID & $69.07\pm 0.49$ & $\underline{41.67\pm 1.15}$ & $4.10$ & \cellcolor{yellow!25}$\boldsymbol{2.15\,\mathrm{h}\,(\times 1.84)}$ & $51.16\pm 1.60$ & $27.78\pm 2.53$ & $17.12$ & $\underline{3.25\,\mathrm{h}\,(\times 3.82)}$ \\
 &  & \textbf{ActFirst-OPD (Ours)} & $69.34\pm 0.45$ & $\boldsymbol{42.00\pm 1.00}$ & $4.18$ & $\underline{2.32\,\mathrm{h}\,(\times 1.71)}$ & $\boldsymbol{53.56\pm 1.25}$ & $\boldsymbol{31.33\pm 1.34}$ & $16.29$ & \cellcolor{yellow!25}$\boldsymbol{2.63\,\mathrm{h}\,(\times 4.72)}$ \\
\bottomrule
\end{tabular}%
}
\end{table}

\subsection{Ablation Studies}
We evaluate three ablation designs.
\textbf{Ours w/o ID} replaces reference-conditioned inverse dynamics
with direct action generation without reference next observations.
\textbf{w/o NAP Fallback} terminates rollouts after failed transition consistency checks instead of continuing with autonomous next-action prediction (NAP).
\textbf{Reference-Prefix Replay} executes either the first 50\% of each reference action sequence, followed by NAP, or the full sequence.
These variants examine reference conditioning, continuation after
divergence, and the use of reference actions for rollout collection,
respectively.

Table~\ref{tab:ablation_average_across_scales} reports results
averaged across student sizes.
ActFirst-OPD has the highest mean SR and lowest mean SR rank
on all three benchmarks.
Removing ID lowers mean SR by 17.31, 1.00, and 3.04 percentage
points on ALFWorld, WebShop, and ScienceWorld, respectively;
training time increases from 1.94 to 2.45 hours on ALFWorld
and from 2.27 to 2.65 hours on ScienceWorld.
Removing NAP fallback lowers mean SR by 1.05, 3.11,
and 2.82 percentage points, respectively, suggesting that
contexts collected after divergence remain useful for distillation.

Both reference-replay variants also lower mean SR on all three
benchmarks.
Compared with ActFirst-OPD, the 50\% variant increases training
time from 1.83 to 2.18 hours on WebShop and from 2.27 to 2.87 hours
on ScienceWorld.
Full replay reduces mean training time on all three benchmarks
but lowers mean SR by 0.36 percentage points on ALFWorld,
with larger drops of 1.89 and 4.22 points on WebShop and
ScienceWorld, respectively.
The larger gaps on WebShop and ScienceWorld may reflect a
greater benefit from learning to handle deviations from
reference trajectories.
Full replay follows fixed reference actions, whereas ActFirst-OPD
continues after divergence, collecting interaction contexts
that support distillation on feedback from student-generated actions.

\newcommand{\SRdrop}[2]{%
    \ensuremath{#1_{\textcolor{red!75!black}{\downarrow #2}}}}

\begin{table}[t]
\centering
\caption{\textbf{Ablation results averaged across Qwen3-0.6B, 1.7B, and 4B.}
Red subscripts show SR drops from ActFirst-OPD in percentage points.
Time denotes training wall-clock time.
Rank averages per-size SR ranks, using average ranks for ties.
\textbf{Bold} and \underline{underlined} values mark the best and
second-best results per benchmark and metric, respectively.}
\label{tab:ablation_average_across_scales}
\footnotesize
\setlength{\tabcolsep}{2.2pt}
\renewcommand{\arraystretch}{1.15}
\resizebox{\linewidth}{!}{%
\begin{tabular}{@{}lccc@{\hspace{7pt}}ccc@{\hspace{7pt}}ccc@{}}
\toprule
\multicolumn{1}{c}{\multirow{2}{*}{Method}}
& \multicolumn{3}{c}{ALFWorld}
& \multicolumn{3}{c}{WebShop}
& \multicolumn{3}{c}{ScienceWorld} \\
\cmidrule(lr){2-4}
\cmidrule(lr){5-7}
\cmidrule(lr){8-10}
& Avg. SR$\uparrow$ & Time (h)$\downarrow$ & Rank$\downarrow$
& Avg. SR$\uparrow$ & Time (h)$\downarrow$ & Rank$\downarrow$
& Avg. SR$\uparrow$ & Time (h)$\downarrow$ & Rank$\downarrow$ \\
\midrule

w/o ID
& \SRdrop{58.92}{17.31} & 2.45 & 4.67
& \SRdrop{38.78}{1.00}  & \underline{1.75} & 3.17
& \SRdrop{19.63}{3.04}  & 2.65 & 3.33 \\

w/o NAP Fallback
& \SRdrop{75.18}{1.05} & \underline{1.25} & 2.33
& \SRdrop{36.67}{3.11} & 2.09 & 4.50
& \SRdrop{\underline{19.85}}{2.82}
  & 2.33 & \underline{3.00} \\

Reference-Prefix Replay (50\%)
& \SRdrop{68.94}{7.29} & 1.90 & 4.33
& \SRdrop{\underline{39.11}}{0.67}
  & 2.18 & \underline{2.67}
& \SRdrop{18.89}{3.78} & 2.87 & 3.67 \\

Reference-Prefix Replay (100\%)
& \SRdrop{\underline{75.87}}{0.36}
  & \textbf{1.11} & \underline{2.00}
& \SRdrop{37.89}{1.89} & \textbf{1.66} & 3.17
& \SRdrop{18.45}{4.22} & \textbf{1.73} & 4.00 \\

\midrule
\textbf{ActFirst-OPD (Ours)}
& \textbf{76.23} & 1.94 & \textbf{1.67}
& \textbf{39.78} & 1.83 & \textbf{1.50}
& \textbf{22.67} & \underline{2.27} & \textbf{1.00} \\

\bottomrule
\end{tabular}%
}
\end{table}

\begin{figure}[t]
    \centering
    \includegraphics[width=\linewidth]{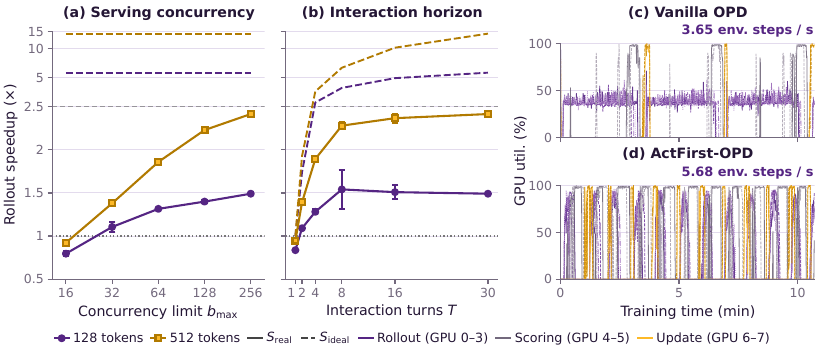}
    \caption{\textbf{Rollout efficiency on ALFWorld with Qwen3-1.7B.}
    (a,b) Measured rollout speedup (solid; mean $\pm$ standard deviation
    over three paired repeats) and idealized estimates (dashed),
    varying $b_{\max}$ at $T=30$ in (a) and $T$ at $b_{\max}=256$ in (b).
    Speedup axes are linear up to $2.5$ and logarithmic above.
    (c,d) GPU utilization for Vanilla OPD and ActFirst-OPD;
    annotations show full-training environment-transition throughput.}
    \label{fig:efficiency-analysis}
\end{figure}
\subsection{Rollout Efficiency and Quality Analysis}
\paragraph{Where Does the Speedup Come From?}
\label{sec:speedup-analysis}
The observed training speedup comes from both higher environment-transition throughput and fewer environment transitions.
Asynchronous full-response generation (Section~\ref{sec:asynchronous-full-response}) can increase throughput, while reference-conditioned inverse dynamics can improve task progress and reduce inefficient interactions.

To isolate the benefit of asynchronous generation, we use fixed ALFWorld interaction contexts with Qwen3-1.7B, 16 tasks per batch, 16-token fast actions, and 128- or 512-token full responses.
We vary the cap $b_{\max}$ on concurrent generation requests,
shared by fast-action and full-response requests,
and the number of virtual turns per task $T$, timing each rollout until all requests finish.
Rollout speedup is the think-then-act completion time divided by the ActFirst-OPD completion time.
Figure~\ref{fig:efficiency-analysis}(a) shows that increasing $b_{\max}$ from 16 to 256 raises speedup from $0.80\times$ to
$1.49\times$ for 128-token responses and from $0.92\times$ to
$2.42\times$ for 512-token responses.
At $b_{\max}=16$, both response lengths yield speedups below one, indicating that overlap does not offset the additional fast-action generation cost in this setting.
Increasing $b_{\max}$ enables speedups in this experiment,
but the required serving concurrency depends on the workload
and hardware.
The larger full-response to fast-action token ratio
(32 versus 8) provides more overlap opportunities and reduces
the relative decoding overhead of fast actions.
At $b_{\max}=256$, speedup initially grows with $T$ and then
levels off (Figure~\ref{fig:efficiency-analysis}(b)).
The initial increase and diminishing returns are consistent with amortizing the final
full-response cost in Proposition~\ref{prop:rollout-speedup}.
Limited serving concurrency can further constrain speedup
and contribute to the gap from the idealized estimates
(Appendix~\ref{app:finite-capacity}).

In 250-update training runs with matched rollout batch size and serving configuration, ActFirst-OPD increases environment-transition throughput from 3.65 to 5.68 transitions per second ($1.56\times$; Figure~\ref{fig:efficiency-analysis}(c,d)).
Across the full-training measurement windows, it also executes 30.45\% fewer transitions (38,616 versus 55,525; Table~\ref{tab:throughput-measurements}).
Since elapsed time equals transition count divided by throughput, the speedup over these windows is
$1.56/(1-0.3045)\approx2.24$.
This is close to the $2.28\times$ training speedup for Qwen3-1.7B
in Table~\ref{tab:main-alfworld}, measured from process launch
to exit in separate runs.

\begin{figure}[t]
    \centering
    \includegraphics[width=\linewidth]{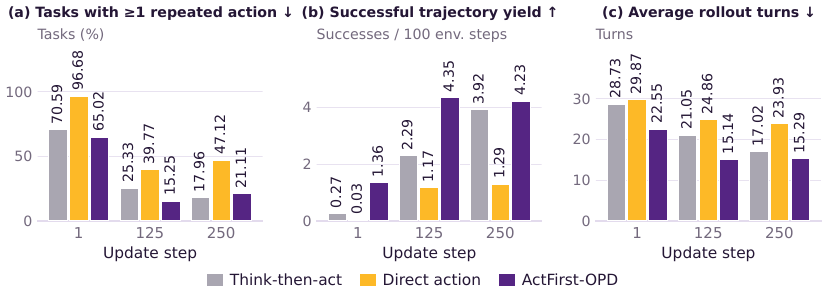}
    \caption{\textbf{Rollout quality on ALFWorld.}
    At each update, all strategies use the same frozen Qwen3-1.7B
    checkpoint on 3,553 training tasks with a 30-turn maximum rollout horizon.}
    \label{fig:rollout-quality}
\end{figure}
\paragraph{How Does ActFirst-OPD Affect Rollout Quality?}
We hold the student checkpoint fixed and vary only the rollout strategy: think-then-act, direct action without reference next observations, or ActFirst-OPD's reference-conditioned interaction with autonomous fallback.
At each of updates 1, 125, and 250, all strategies use the same frozen
Qwen3-1.7B Vanilla OPD checkpoint on all 3,553 ALFWorld training
tasks with a 30-turn limit.
We measure the fraction of tasks with at least one unproductive action repetition, successful trajectories per 100 environment transitions (successful trajectory yield), and mean rollout turns. 
Measurement protocols are detailed in Appendix~\ref{app:analysis-protocols}.

Compared with direct action, ActFirst-OPD improves rollout quality at every checkpoint, with a lower fraction of tasks exhibiting unproductive action repetition, higher successful trajectory yield, and fewer rollout turns (Figure~\ref{fig:rollout-quality}).
Compared with think-then-act, it yields more successful trajectories per 100 transitions and uses fewer turns, with broadly comparable task-level repetition.
Based on these results, reference-conditioned inverse dynamics may improve rollout quality.

\section{Conclusion}
We presented ActFirst-OPD, which combines reference-conditioned inverse dynamics, autonomous next-action prediction, and asynchronous full-response generation.
Across three benchmarks and three Qwen3 student sizes, it reduces training time relative to Vanilla OPD while matching or exceeding all compared OPD baselines in mean success rate across eight of nine settings.
These results highlight that reasoning need not block acting during multi-turn agent distillation.

Despite these gains, our method assumes access to high-quality offline reference trajectories; settings where such references are difficult to obtain are beyond the scope of this work.
Reference next observations currently serve only as rollout guidance, which is disabled after divergence.
Future work could use reference transitions for auxiliary supervision or realign rollouts with their reference trajectories after divergence.

\subsection*{AI use statement}

The authors used OpenAI language-model systems, including Codex,
as interactive assistants during this project.
These tools assisted with literature discovery, experimental design,
code development, data processing and visualization, mathematical
derivations and proof checking, and manuscript drafting and revision.
The human authors made the final research and methodological decisions
and take full responsibility for the manuscript, experimental results,
and theoretical claims.

\subsection*{Ethics statement}

This work studies efficient training of language agents using
established benchmarks: ALFWorld, WebShop, and ScienceWorld.
Our experiments involve benchmark environments and pretrained
language models, without human-participant studies or real-world
deployment.
Reducing training costs can broaden access to agent research,
but may also lower barriers to developing agents for harmful uses.
Student models may inherit biases, errors, or undesirable behaviors
from their teachers and reference trajectories.
Applications beyond the evaluated benchmarks therefore require
appropriate safety evaluation and controls on environment
and tool access.

\subsection*{Reproducibility statement}

Section~\ref{sec:experiments} describes the benchmarks, models,
computational resources, and evaluation protocols.
The assumptions and proofs underlying our efficiency analysis
are provided in Appendix~\ref{app:theoretical-analysis}.
Appendix~\ref{app:implementation-details} provides the prompt
templates, history management, thinking-budget decoding, OPD optimization,
and transition consistency checks.
Appendix~\ref{app:experiment-details} describes benchmark configurations,
teacher training and reference construction, training and evaluation
hyperparameters, and rollout measurement protocols.
The accompanying anonymous code repository is available at
\url{https://anonymous.4open.science/r/ActFirst-OPD}.



\bibliography{iclr2027_conference}

@article{thinkingmachineslab-opd,
  author       = {Kevin Lu and {Thinking Machines Lab}},
  title        = {On-Policy Distillation},
  year         = {2025},
  howpublished = {Thinking Machines Lab: Connectionism},
  note         = {https://thinkingmachines.ai/blog/on-policy-distillation}
}

@inproceedings{react,
  author       = {Shunyu Yao and
                  Jeffrey Zhao and
                  Dian Yu and
                  Nan Du and
                  Izhak Shafran and
                  Karthik R. Narasimhan and
                  Yuan Cao},
  title        = {{ReAct}: Synergizing Reasoning and Acting in Language Models},
  booktitle    = {International Conference on Learning Representations (ICLR)},
  year         = {2023}
}

@inproceedings{alfworld,
  author       = {Mohit Shridhar and
                  Xingdi Yuan and
                  Marc{-}Alexandre C{\^{o}}t{\'{e}} and
                  Yonatan Bisk and
                  Adam Trischler and
                  Matthew J. Hausknecht},
  title        = {{ALFWorld}: Aligning Text and Embodied Environments for Interactive
                  Learning},
  booktitle    = {International Conference on Learning Representations (ICLR)},
  year         = {2021}
}

@inproceedings{webshop,
  author       = {Shunyu Yao and
                  Howard Chen and
                  John Yang and
                  Karthik Narasimhan},
  title        = {{WebShop}: Towards Scalable Real-World Web Interaction with Grounded
                  Language Agents},
  booktitle    = {Advances in Neural Information Processing Systems (NeurIPS)},
  volume       = {35},
  year         = {2022}
}

@inproceedings{scienceworld,
  author       = {Ruoyao Wang and
                  Peter A. Jansen and
                  Marc{-}Alexandre C{\^{o}}t{\'{e}} and
                  Prithviraj Ammanabrolu},
  title        = {{ScienceWorld}: Is your Agent Smarter than a 5th Grader?},
  booktitle    = {Proceedings of the 2022 Conference on Empirical Methods in Natural Language Processing (EMNLP)},
  pages        = {11279--11298},
  year         = {2022}
}

@inproceedings{minillm,
  author       = {Yuxian Gu and
                  Li Dong and
                  Furu Wei and
                  Minlie Huang},
  title        = {{MiniLLM}: Knowledge Distillation of Large Language Models},
  booktitle    = {International Conference on Learning Representations (ICLR)},
  year         = {2024}
}

@inproceedings{gkd,
  author       = {Rishabh Agarwal and
                  Nino Vieillard and
                  Yongchao Zhou and
                  Piotr Stanczyk and
                  Sabela Ramos Garea and
                  Matthieu Geist and
                  Olivier Bachem},
  title        = {On-Policy Distillation of Language Models: Learning from Self-Generated
                  Mistakes},
  booktitle    = {International Conference on Learning Representations (ICLR)},
  year         = {2024}
}

@inproceedings{tcod,
  author    = {Jiaqi Wang and
                  Wenhao Zhang and
                  Weijie Shi and
                  Yaliang Li and
                  James Cheng},
  title     = {Exploring Temporal Curriculum in On-Policy Distillation for Multi-turn Autonomous Agents},
  booktitle = {Proceedings of the 3rd Conference on Language Modeling (COLM)},
  year      = {2026}
}

@article{turnopd,
  author       = {Yuhang Zhou and
                  Kai Zheng and
                  Haoling Li and
                  Dengyun Peng and
                  Can Xu and
                  Jingjing Chen},
  title        = {{TurnOPD}: Making On-Policy Distillation Turn-Aware for Efficient Long-Horizon
                  Agent Training},
  journal      = {arXiv preprint arXiv:2607.05804},
  year         = {2026}
}

@article{reopd,
  author       = {Baohao Liao and
                  Hanze Dong and
                  Christof Monz and
                  Xinxing Xu and
                  Li Dong and
                  Furu Wei},
  title        = {Multi-Turn On-Policy Distillation with Prefix Replay},
  journal      = {arXiv preprint arXiv:2607.04763},
  year         = {2026},
}

@article{ridm,
  author       = {Brahma S. Pavse and
                  Faraz Torabi and
                  Josiah Hanna and
                  Garrett Warnell and
                  Peter Stone},
  title        = {{RIDM}: Reinforced Inverse Dynamics Modeling for Learning from a Single
                  Observed Demonstration},
  journal      = {{IEEE} Robotics and Automation Letters},
  volume       = {5},
  number       = {4},
  pages        = {6262--6269},
  year         = {2020}
}

@article{rethinking-opd,
  author       = {Yaxuan Li and
                  Yuxin Zuo and
                  Bingxiang He and
                  Jinqian Zhang and
                  Chaojun Xiao and
                  Cheng Qian and
                  Tianyu Yu and
                  Huan{-}ang Gao and
                  Wenkai Yang and
                  Zhiyuan Liu and
                  Ning Ding},
  title        = {Rethinking On-Policy Distillation of Large Language Models: Phenomenology,
                  Mechanism, and Recipe},
  journal      = {arXiv preprint arXiv:2604.13016},
  year         = {2026}
}

@article{qwen3,
  author       = {{Qwen Team}},
  title        = {{Qwen3} Technical Report},
  journal      = {arXiv preprint arXiv:2505.09388},
  year         = {2025}
}

@article{deepseekv4,
  author       = {{DeepSeek{-}AI}},
  title        = {{DeepSeek-V4}: Towards Highly Efficient Million-Token Context Intelligence},
  journal      = {arXiv preprint arXiv:2606.19348},
  year         = {2026}
}

@article{opd-survey,
  author       = {Mingyang Song and
                  Mao Zheng},
  title        = {A Survey of On-Policy Distillation for Large Language Models},
  journal      = {arXiv preprint arXiv:2604.00626},
  year         = {2026}
}

@inproceedings{gigpo,
  author       = {Lang Feng and
                  Zhenghai Xue and
                  Tingcong Liu and
                  Bo An},
  title        = {Group-in-Group Policy Optimization for {LLM} Agent Training},
  booktitle    = {Advances in Neural Information Processing Systems (NeurIPS)},
  volume       = {38},
  year         = {2025}
}

@inproceedings{simpletir,
  author       = {Zhenghai Xue and
                  Longtao Zheng and
                  Qian Liu and
                  Yingru Li and
                  Xiaosen Zheng and
                  Zejun Ma and
                  Bo An},
  title        = {{SimpleTIR}: End-to-End Reinforcement Learning for Multi-Turn Tool-Integrated
                  Reasoning},
  booktitle={International Conference on Learning Representations (ICLR)},
  year={2026}
}

@article{codeasharness,
  author       = {Xuying Ning and
                  Katherine Tieu and
                  Dongqi Fu and
                  Tianxin Wei and
                  Zihao Li and
                  Yuanchen Bei and
                  Jiaru Zou and
                  Mengting Ai and
                  Zhining Liu and
                  Ting{-}Wei Li and
                  Lingjie Chen and
                  Yanjun Zhao and
                  Ke Yang and
                  Bingxuan Li and
                  Cheng Qian and
                  Gaotang Li and
                  Xiao Lin and
                  Zhichen Zeng and
                  Ruizhong Qiu and
                  Sirui Chen and
                  Yifan Sun and
                  Xiyuan Yang and
                  Ruida Wang and
                  Rui Pan and
                  Chenyuan Yang and
                  Dylan Zhang and
                  Liri Fang and
                  Zikun Cui and
                  Yang Cao and
                  Pan Chen and
                  Dorothy Sun and
                  Ren Chen and
                  Mahesh Srinivasan and
                  Nipun Mathur and
                  Yinglong Xia and
                  Hong Li and
                  Hong Yan and
                  Pan Lu and
                  Lingming Zhang and
                  Tong Zhang and
                  Hanghang Tong and
                  Jingrui He},
  title        = {Code as Agent Harness},
  journal      = {arXiv preprint arXiv:2605.18747},
  year         = {2026}
}

@article{claudecode,
  author       = {{Anthropic}},
  title        = {{Claude 3.7 Sonnet and Claude Code}},
  year         = {2025},
  note          = {https://www.anthropic.com/news/claude-3-7-sonnet}
}

@article{codex,
  author       = {{OpenAI}},
  title        = {Introducing {Codex}},
  year         = {2025},
  note          = {https://openai.com/index/introducing-codex/}
}

@article{agentharness,
  author       = {{Anthropic}},
  title        = {Effective Harnesses for Long-Running Agents},
  year         = {2025},
  note          = {https://www.anthropic.com/engineering/effective-harnesses-for-long-running-agents}
}

@inproceedings{pathak2017curiosity,
  author       = {Deepak Pathak and
                  Pulkit Agrawal and
                  Alexei A. Efros and
                  Trevor Darrell},
  title        = {Curiosity-driven Exploration by Self-supervised Prediction},
  booktitle    = {Proceedings of the 34th International Conference on Machine Learning (ICML)},
  volume       = {70},
  pages        = {2778--2787},
  year         = {2017}
}

@inproceedings{bco,
  author       = {Faraz Torabi and
                  Garrett Warnell and
                  Peter Stone},
  title        = {Behavioral Cloning from Observation},
  booktitle    = {Proceedings of the 27th International Joint Conference on
                  Artificial Intelligence (IJCAI)},
  pages        = {4950--4957},
  year         = {2018}
}

@inproceedings{vpp,
  author       = {Yucheng Hu and
                  Yanjiang Guo and
                  Pengchao Wang and
                  Xiaoyu Chen and
                  Yen{-}Jen Wang and
                  Jianke Zhang and
                  Koushil Sreenath and
                  Chaochao Lu and
                  Jianyu Chen},
  title        = {{Video Prediction Policy}: A Generalist Robot Policy with Predictive
                  Visual Representations},
  booktitle    = {Proceedings of the 42nd International Conference on Machine Learning (ICML)},
  volume       = {267},
  pages        = {24328--24346},
  year         = {2025}
}

@article{vera,
  author       = {Sizhe Lester Li and
                  Evan Kim and
                  Xingjian Bai and
                  Tong Zhao and
                  Tao Pang and
                  Max Simchowitz and
                  Vincent Sitzmann},
  title        = {Turning Video Models into Generalist Robot Policies},
  journal      = {arXiv preprint arXiv:2605.27817},
  year         = {2026}
}

@article{trinityrft,
  author       = {Xuchen Pan and
                  Yanxi Chen and
                  Yushuo Chen and
                  Yuchang Sun and
                  Daoyuan Chen and
                  Wenhao Zhang and
                  Yuexiang Xie and
                  Yilun Huang and
                  Yilei Zhang and
                  Dawei Gao and
                  Weijie Shi and
                  Yaliang Li and
                  Bolin Ding and
                  Jingren Zhou},
  title        = {{Trinity-RFT}: A General-Purpose and Unified Framework for Reinforcement
                  Fine-Tuning of Large Language Models},
  journal      = {arXiv preprint arXiv:2505.17826},
  year         = {2025}
}

@inproceedings{vllm,
  author       = {Woosuk Kwon and
                  Zhuohan Li and
                  Siyuan Zhuang and
                  Ying Sheng and
                  Lianmin Zheng and
                  Cody Hao Yu and
                  Joseph E. Gonzalez and
                  Hao Zhang and
                  Ion Stoica},
  title        = {Efficient Memory Management for Large Language Model Serving with
                  {PagedAttention}},
  booktitle    = {Proceedings of the 29th Symposium on Operating Systems Principles (SOSP)},
  pages        = {611--626},
  year         = {2023}
}

@inproceedings{verl,
  author       = {Guangming Sheng and
                  Chi Zhang and
                  Zilingfeng Ye and
                  Xibin Wu and
                  Wang Zhang and
                  Ru Zhang and
                  Yanghua Peng and
                  Haibin Lin and
                  Chuan Wu},
  title        = {{HybridFlow}: A Flexible and Efficient {RLHF} Framework},
  booktitle    = {Proceedings of the 20th European Conference on Computer Systems (EuroSys)},
  pages        = {1279--1297},
  year         = {2025}
}

@inproceedings{adamw,
  author       = {Ilya Loshchilov and
                  Frank Hutter},
  title        = {Decoupled Weight Decay Regularization},
  booktitle    = {International Conference on Learning Representations (ICLR)},
  year         = {2019}
}

@article{ppo,
  author       = {John Schulman and
                  Filip Wolski and
                  Prafulla Dhariwal and
                  Alec Radford and
                  Oleg Klimov},
  title        = {Proximal Policy Optimization Algorithms},
  journal      = {arXiv preprint arXiv:1707.06347},
  year         = {2017},
}

@inproceedings{dualclip,
  author       = {Deheng Ye and
                  Zhao Liu and
                  Mingfei Sun and
                  Bei Shi and
                  Peilin Zhao and
                  Hao Wu and
                  Hongsheng Yu and
                  Shaojie Yang and
                  Xipeng Wu and
                  Qingwei Guo and
                  Qiaobo Chen and
                  Yinyuting Yin and
                  Hao Zhang and
                  Tengfei Shi and
                  Liang Wang and
                  Qiang Fu and
                  Wei Yang and
                  Lanxiao Huang},
  title        = {Mastering Complex Control in {MOBA} Games with Deep Reinforcement
                  Learning},
  booktitle    = {Proceedings of the 34th {AAAI} Conference on Artificial Intelligence (AAAI)},
  pages        = {6672--6679},
  year         = {2020}
}

@article{ulysses,
  author       = {Sam Ade Jacobs and
                  Masahiro Tanaka and
                  Chengming Zhang and
                  Minjia Zhang and
                  Shuaiwen Leon Song and
                  Samyam Rajbhandari and
                  Yuxiong He},
  title        = {{DeepSpeed Ulysses}: System Optimizations for Enabling Training of Extreme
                  Long Sequence Transformer Models},
  journal      = {arXiv preprint arXiv:2309.14509},
  year         = {2023}
}

@inproceedings{textworld,
  author       = {C{\^o}t{\'e}, Marc-Alexandre and
                  K{\'a}d{\'a}r, {\'A}kos and
                  Yuan, Xingdi and
                  Kybartas, Ben and
                  Barnes, Tavian and
                  Fine, Emery and
                  Moore, James and
                  Hausknecht, Matthew and
                  El Asri, Layla and
                  Adada, Mahmoud and
                  Tay, Wendy and
                  Trischler, Adam},
  title        = {{TextWorld}: A Learning Environment for Text-Based Games},
  booktitle    = {Computer Games},
  series       = {Communications in Computer and Information Science},
  volume       = {1017},
  pages        = {41--75},
  year         = {2019}
}

@article{sageopd,
  author       = {Yuhang Zhou and
                  Lizhu Zhang and
                  Yifan Wu and
                  Mingyi Wang and
                  Bo Peng and
                  Jiayi Liu and
                  Xiangjun Fan and
                  Zhuokai Zhao},
  title        = {{SAGE-OPD}: Selective Agent-Guided Intervention for Multi-Turn On-Policy Distillation},
  journal      = {arXiv preprint arXiv:2606.19659},
  year         = {2026}
}

@article{guidedopd,
  author       = {Gengsheng Li and
                  Mao Zheng and
                  Mingyang Song and
                  Ruiqi Liu and
                  Tianyu Yang and
                  Jie Sun and
                  Qiyong Zhong and
                  Haiyun Guo and
                  Junfeng Fang and
                  Dan Zhang and
                  Jinqiao Wang},
  title        = {On-Policy Distillation with Curriculum Turn-level Guidance for Multi-turn Agents},
  journal      = {arXiv preprint arXiv:2606.15912},
  year         = {2026}
}

@article{npd,
  author       = {Miao Rang and
                  Zhenni Bi and
                  Hang Zhou and
                  Kai Han and
                  Xuechun Wang and
                  An Xiao and
                  Xinghao Chen and
                  Yunhe Wang and
                  Hanting Chen},
  title        = {Near-Policy: Accelerating On-Policy Distillation via Asynchronous Generation and Selective Packing},
  journal      = {arXiv preprint arXiv:2605.05940},
  year         = {2026}
}

@inproceedings{sad,
  author       = {Jun Liu and
                  Zhenglun Kong and
                  Peiyan Dong and
                  Changdi Yang and
                  Tianqin Li and
                  Yanyue Xie and
                  Yifan Gong and
                  Xuan Shen and
                  Pu Zhao and
                  Hao Tang and
                  Geng Yuan and
                  Wei Niu and
                  Wenbin Zhang and
                  Xue Lin and
                  Dong Huang and
                  Yanzhi Wang},
  title        = {Structured Agent Distillation for Large Language Model Agents},
  booktitle    = {Proceedings of the 25th International Conference on Autonomous Agents and Multiagent Systems (AAMAS)},
  pages        = {3676--3685},
  year         = {2026}
}

@article{lightningopd,
  author       = {Yecheng Wu and
                  Song Han and
                  Hai Cai},
  title        = {Lightning {OPD}: Efficient Post-Training for Large Reasoning Models with Offline On-Policy Distillation},
  journal      = {arXiv preprint arXiv:2604.13010},
  year         = {2026},
}

@article{mopd,
  author       = {Weichen Yu and
                  Xiaomin Li and
                  Yizhou Zhao and
                  Xiaoze Liu and
                  Ruowang Zhang and
                  Haixin Wang and
                  Yinyi Luo and
                  Chen Henry Wu and
                  Gaurav Mittal and
                  Matt Fredrikson and
                  Yu Hu},
  title        = {Multi-Rollout On-Policy Distillation via Peer Successes and Failures},
  journal      = {arXiv preprint arXiv:2605.12652},
  year         = {2026}
}
\bibliographystyle{iclr2027_conference}

\clearpage

\appendix

\section{Theoretical Analysis}
\label{app:theoretical-analysis}

\subsection{Proof of Proposition~\ref{prop:rollout-speedup}}
\label{app:proof-rollout-speedup}

\begin{proof}
We measure time from the availability of the initial context $\vc_1$.
Throughout the proof, $T\ge 1$ and
$0<\ell_{\mathrm{fast}}\le\ell_{\mathrm{full}}$.
Under the assumptions of Proposition~\ref{prop:rollout-speedup},
environment transitions and other non-generation operations take
negligible time, and overlapping requests do not increase their
generation latencies.

In standard think-then-act interaction, the next context becomes
available only after the current full response has been generated
and its action executed.
Since each response takes $\ell_{\mathrm{full}}$, completing
$T$ turns requires
\begin{equation}
    C_{\mathrm{Vanilla}}
    =
    T\ell_{\mathrm{full}}.
    \label{eq:proof-vanilla-time}
\end{equation}

For ActFirst-OPD, let $s_t$ denote the time at which context $\vc_t$
becomes available.
Both the fast-action request and the full-response request start
at $s_t$.
Only the fast action is required to obtain the next context, giving
\begin{equation}
    s_1=0,
    \qquad
    s_{t+1}=s_t+\ell_{\mathrm{fast}}
    \quad (1\le t<T).
\end{equation}
It follows that
\begin{equation}
    s_t=(t-1)\ell_{\mathrm{fast}}.
    \label{eq:proof-context-release}
\end{equation}
Environment interaction finishes when the final fast action is
executed, at time $T\ell_{\mathrm{fast}}$.
The full response for turn $t$ finishes at
$s_t+\ell_{\mathrm{full}}$.
Thus, completing both interaction and all full responses requires
\begin{align}
    C_{\mathrm{ActFirst}}^{\mathrm{ideal}}
    &=
    \max\left\{
        T\ell_{\mathrm{fast}},
        \max_{1\le t\le T}
        \bigl(s_t+\ell_{\mathrm{full}}\bigr)
    \right\}
    \nonumber\\
    &=
    \max\left\{
        T\ell_{\mathrm{fast}},
        (T-1)\ell_{\mathrm{fast}}+\ell_{\mathrm{full}}
    \right\}
    \nonumber\\
    &=
    (T-1)\ell_{\mathrm{fast}}+\ell_{\mathrm{full}},
    \label{eq:proof-actfirst-time}
\end{align}
where the last equality follows from
$\ell_{\mathrm{full}}\ge\ell_{\mathrm{fast}}$.

Taking the ratio of
Eqs.~\ref{eq:proof-vanilla-time}
and~\ref{eq:proof-actfirst-time} yields
\begin{equation}
    S_{\mathrm{ideal}}(T)
    =
    \frac{T\ell_{\mathrm{full}}}
    {(T-1)\ell_{\mathrm{fast}}+\ell_{\mathrm{full}}}.
    \label{eq:proof-speedup}
\end{equation}

Since $T\ge 1$, the denominator satisfies
\begin{equation}
    (T-1)\ell_{\mathrm{fast}}+\ell_{\mathrm{full}}
    \ge \ell_{\mathrm{full}},
\end{equation}
which gives
\begin{equation}
    S_{\mathrm{ideal}}(T)
    \le
    \frac{T\ell_{\mathrm{full}}}{\ell_{\mathrm{full}}}
    = T.
    \label{eq:proof-turn-bound}
\end{equation}
Moreover, the assumption
$\ell_{\mathrm{full}}\ge\ell_{\mathrm{fast}}$ implies
\begin{equation}
    (T-1)\ell_{\mathrm{fast}}+\ell_{\mathrm{full}}
    \ge
    (T-1)\ell_{\mathrm{fast}}+\ell_{\mathrm{fast}}
    =
    T\ell_{\mathrm{fast}}.
\end{equation}
Therefore,
\begin{equation}
    S_{\mathrm{ideal}}(T)
    \le
    \frac{T\ell_{\mathrm{full}}}{T\ell_{\mathrm{fast}}}
    =
    \frac{\ell_{\mathrm{full}}}{\ell_{\mathrm{fast}}}.
    \label{eq:proof-latency-bound}
\end{equation}
The denominator in Eq.~\ref{eq:proof-speedup} is also at most
$T\ell_{\mathrm{full}}$, so $S_{\mathrm{ideal}}(T)\ge 1$.
Combining this lower bound with the upper bounds in
Eqs.~\ref{eq:proof-turn-bound} and~\ref{eq:proof-latency-bound}
gives the following range for the idealized rollout speedup:
\begin{equation}
    1
    \le
    S_{\mathrm{ideal}}(T)
    \le
    \min\left\{
        T,\,
        \frac{\ell_{\mathrm{full}}}{\ell_{\mathrm{fast}}}
    \right\}.
\end{equation}

Finally, rewriting Eq.~\ref{eq:proof-speedup} gives
\begin{equation}
    S_{\mathrm{ideal}}(T)
    =
    \left(
        \frac{\ell_{\mathrm{fast}}}{\ell_{\mathrm{full}}}
        +
        \frac{
            1-\ell_{\mathrm{fast}}/\ell_{\mathrm{full}}
        }{T}
    \right)^{-1}.
\end{equation}
Since
$0<\ell_{\mathrm{fast}}/\ell_{\mathrm{full}}\le 1$,
the expression inside parentheses is nonincreasing in $T$.
Hence, $S_{\mathrm{ideal}}(T)$ is nondecreasing in $T$, with
\begin{equation}
    \lim_{T\to\infty}S_{\mathrm{ideal}}(T)
    =
    \frac{\ell_{\mathrm{full}}}{\ell_{\mathrm{fast}}}.
    \label{eq:ideal-rollout-speedup-limit}
\end{equation}
\end{proof}

\subsection{Finite-Concurrency Constraints}
\label{app:finite-capacity}

We analyze how serving concurrency, full-response length,
and rollout horizon constrain realized rollout speedup
under fixed generation workloads.
Consider $N$ tasks, each containing $T$ turns, served by a
fixed pool of inference resources.
Let $b_{\max}$ denote the cap on concurrent generation requests,
shared by fast-action and full-response requests across this pool.
All initial contexts are available at time zero, and environment
transitions and other non-generation operations take negligible time.
At each turn, Vanilla OPD generates one full response before acting.
ActFirst-OPD generates one fast action for environment interaction
and asynchronously generates one full response from the
interaction context.

Each admitted request occupies one concurrency slot until completion.
We assume that $0<\ell_{\mathrm{fast}}\le\ell_{\mathrm{full}}$
lower-bound the corresponding times spent holding a slot,
including prefill and decoding but excluding waiting for admission.
Both methods use the same resources and concurrency cap.
All completion times below correspond to the given $b_{\max}$;
this dependence is suppressed in the notation.

\paragraph{Dependency constraint.}
Let $s_{i,t}$ denote the availability time of the context for
task $i$ at turn $t$, with $s_{i,1}=0$.
The next context requires completion of the current fast action, so
\begin{equation}
    s_{i,t+1}
    \ge
    s_{i,t}+\ell_{\mathrm{fast}},
    \qquad
    s_{i,t}\ge(t-1)\ell_{\mathrm{fast}}.
\end{equation}
The final full response cannot finish before
$s_{i,T}+\ell_{\mathrm{full}}$.
Thus, the batch completion time, measured until both environment
interaction and all full responses have finished, satisfies
\begin{equation}
    C_{\mathrm{ActFirst}}
    \ge
    (T-1)\ell_{\mathrm{fast}}+\ell_{\mathrm{full}}.
    \label{eq:capacity-dependency-bound}
\end{equation}

\paragraph{Concurrency constraint.}
Let $n(u)$ denote the number of admitted, unfinished requests
at time $u$, so $n(u)\le b_{\max}$.
ActFirst-OPD generates $NT$ fast actions and $NT$ full responses.
Their total time occupying concurrency slots therefore satisfies
\begin{equation}
    NT(\ell_{\mathrm{fast}}+\ell_{\mathrm{full}})
    \le
    \int_0^{C_{\mathrm{ActFirst}}} n(u)\,du
    \le
    b_{\max}C_{\mathrm{ActFirst}}.
\end{equation}
Consequently,
\begin{equation}
    C_{\mathrm{ActFirst}}
    \ge
    \frac{NT(\ell_{\mathrm{fast}}+\ell_{\mathrm{full}})}
         {b_{\max}}.
    \label{eq:capacity-concurrency-bound}
\end{equation}
Combining Eqs.~\ref{eq:capacity-dependency-bound}
and~\ref{eq:capacity-concurrency-bound} gives
\begin{equation}
    C_{\mathrm{ActFirst}}
    \ge
    \max\left\{
        (T-1)\ell_{\mathrm{fast}}+\ell_{\mathrm{full}},
        \frac{NT(\ell_{\mathrm{fast}}+\ell_{\mathrm{full}})}
             {b_{\max}}
    \right\}.
    \label{eq:finite-capacity-time-bound}
\end{equation}

\paragraph{Implications for speedup.}
Let $C_{\mathrm{Vanilla}}$ denote the completion time of the
corresponding Vanilla OPD rollout batch.
The realized rollout speedup satisfies
\begin{equation}
    S_{\mathrm{real}}
    =
    \frac{C_{\mathrm{Vanilla}}}{C_{\mathrm{ActFirst}}}
    \le
    \min\left\{
        \frac{C_{\mathrm{Vanilla}}}
             {(T-1)\ell_{\mathrm{fast}}+\ell_{\mathrm{full}}},
        \frac{b_{\max}C_{\mathrm{Vanilla}}}
             {NT(\ell_{\mathrm{fast}}+\ell_{\mathrm{full}})}
    \right\}.
    \label{eq:finite-capacity-speedup}
\end{equation}
This bound retains the actual baseline completion time under
the same concurrency cap.

When Vanilla OPD attains the ideal completion time
$C_{\mathrm{Vanilla}}=T\ell_{\mathrm{full}}$,
Eq.~\ref{eq:finite-capacity-speedup} simplifies to
\begin{equation}
    S_{\mathrm{real}}
    \le
    \min\left\{
        S_{\mathrm{ideal}}(T),
        \frac{b_{\max}}
             {N(1+\ell_{\mathrm{fast}}/\ell_{\mathrm{full}})}
    \right\}.
    \label{eq:concurrency-speedup-bound}
\end{equation}
Under this baseline assumption, a necessary condition for
$S_{\mathrm{real}}>1$ is
\begin{equation}
    b_{\max}
    >
    N\left(1+\frac{\ell_{\mathrm{fast}}}{\ell_{\mathrm{full}}}\right).
    \label{eq:concurrency-speedup-condition}
\end{equation}
This condition is not sufficient for speedup.
The first term in Eq.~\ref{eq:concurrency-speedup-bound}
captures the dependence on rollout horizon $T$, while the second
limits the gain under a given concurrency cap $b_{\max}$.
Both depend on the latency ratio
$\ell_{\mathrm{fast}}/\ell_{\mathrm{full}}$.
Under this baseline assumption, limited concurrency can prevent
asynchronous overlap from offsetting the additional fast-action
generation cost.
Increasing $b_{\max}$ relaxes this constraint,
but does not guarantee speedup.

Full-response length affects these bounds through
$\ell_{\mathrm{full}}$; no proportionality between token counts
and generation latencies is assumed.
These bounds motivate the controlled analysis in
Section~\ref{sec:speedup-analysis}, which varies $b_{\max}$,
full-response length, and rollout horizon $T$,
with $N=16$ tasks per batch.

\section{Implementation Details}
\label{app:implementation-details}

This appendix details prompt construction, history management, thinking-budget decoding, OPD optimization, and benchmark-specific transition consistency checks.
Agent inputs are constructed from the task instruction $q$, history $h_t$, and current observation $o_t$, using benchmark-specific templates for full-response generation, reference-conditioned inverse dynamics, and autonomous next-action prediction.
Reference next observations are included only in inverse-dynamics requests.

\subsection{Prompt Details}
\label{app:prompts}

For each benchmark, we show the user templates with history for three
generation modes: full-response generation, next-action
prediction, and reference-conditioned inverse dynamics.
Each template includes the task instruction, retained observation--action
history, current observation, and benchmark-specific admissible action information.
Only the inverse-dynamics template includes a reference next observation.

\paragraph{ALFWorld.}\mbox{}\par
\begin{agentprompt}{PromptReason}{ALFWorld: Full-response generation}
You are an expert agent operating in the ALFRED Embodied Environment. Your task is to: {task_description}
Prior to this step, you have already taken {step_count} step(s). Below are the most recent {history_length} observations and the corresponding actions you took: {action_history}
You are now at step {current_step} and your current observation is: {current_observation}
Your admissible actions of the current situation are: [{admissible_actions}].

Now it's your turn to take an action.
You should first reason step-by-step about the current situation. This reasoning process MUST be enclosed within <think> </think> tags.
Once you've finished your reasoning, you should choose an admissible action for current step and present it within <action> </action> tags.
\end{agentprompt}

\begin{agentprompt}{PromptFast}{ALFWorld: Next-action prediction}
You are an expert agent operating in the ALFRED Embodied Environment. Your task is to: {task_description}
Prior to this step, you have already taken {step_count} step(s). Below are the most recent {history_length} observations and the corresponding actions you took: {action_history}
You are now at step {current_step} and your current observation is: {current_observation}
Your admissible actions of the current situation are: [{admissible_actions}].

Now it's your turn to take an action.
You should choose an admissible action for current step and present it within <action> </action> tags.
\end{agentprompt}

\begin{agentprompt}{PromptReference}{ALFWorld: Reference-conditioned inverse dynamics}
You are an expert agent operating in the ALFRED Embodied Environment. Your task is to: {task_description}
Prior to this step, you have already taken {step_count} step(s). Below are the most recent {history_length} observations and the corresponding actions you took: {action_history}
You are now at step {current_step} and your current observation is: {current_observation}

Complete the missing current action in this reference transition:
Current observation --[MISSING CURRENT ACTION]--> reference next observation.

The reference next observation is the environment state immediately after the missing current action. It may be a terminal state.
Reference next observation:
{reference_next_observation}

The reference next observation is the direct result of exactly one action. Do not insert an intermediate or preparatory action.

Your admissible actions of the current situation are: [{admissible_actions}].

Now it's your turn to take one action for the current step. Predict the MISSING CURRENT ACTION that leads to the reference next observation.
You should choose one admissible action and present it within <action> </action> tags.
\end{agentprompt}

\paragraph{WebShop.}\mbox{}\par

\begin{agentprompt}{PromptReason}{WebShop: Full-response generation}
You are an expert autonomous agent operating in the WebShop e-commerce environment.
Your task is to: {task_description}.
Prior to this step, you have already taken {step_count} step(s). Below are the most recent {history_length} observations and the corresponding actions you took: {action_history}
You are now at step {current_step} and your current observation is: {current_observation}.
Your admissible actions of the current situation are:
[
{available_actions}
].

Now it's your turn to take one action for the current step.
You should first reason step-by-step about the current situation, then think carefully which admissible action best advances the shopping goal. This reasoning process MUST be enclosed within <think> </think> tags.
Once you've finished your reasoning, you should choose an admissible action for current step and present it within <action> </action> tags.
\end{agentprompt}

\begin{agentprompt}{PromptFast}{WebShop: Next-action prediction}
You are an expert autonomous agent operating in the WebShop e-commerce environment.
Your task is to: {task_description}.
Prior to this step, you have already taken {step_count} step(s). Below are the recent {history_length} observations and the corresponding actions you took:
{action_history}

You are now at step {current_step} and your current observation is:
{current_observation}

Your admissible actions of the current situation are:
[
{available_actions}
].

Now it's your turn to take one action for the current step.
You should choose an admissible action for current step and present it within <action> </action> tags.
\end{agentprompt}

\begin{agentprompt}{PromptReference}{WebShop: Reference-conditioned inverse dynamics}
You are an expert autonomous agent operating in the WebShop e-commerce environment.
Your task is to: {task_description}.
Prior to this step, you have already taken {step_count} step(s). Below are the recent {history_length} observations and the corresponding actions you took:
{action_history}

You are now at step {current_step} and your current observation is:
{current_observation}

Complete the missing current action in this reference transition:
Current observation --[MISSING CURRENT ACTION]--> reference next observation.

The reference next observation is the environment state immediately after the missing current action. It may be a terminal Score page.
Reference next observation:
{reference_next_observation}

If the missing current action is search[...], reconstruct the query most likely to produce the reference next observation. Do not add price or budget terms to the query.

Your admissible actions of the current situation are:
[
{available_actions}
].

Now it's your turn to take one action for the current step. Predict the MISSING CURRENT ACTION that leads to the reference next observation.
You should choose one admissible action and present it within <action> </action> tags.
\end{agentprompt}

\paragraph{ScienceWorld.}\mbox{}\par

\begin{agentprompt}{PromptReason}{ScienceWorld: Full-response generation}
Your ScienceWorld task is: {task_description}
Prior to this step, you have already taken {step_count} step(s). Below are the most recent {history_length} observations and the corresponding actions you took: {action_history}
You are now at step {current_step} and your current observation is: {current_observation}
Available action commands: [{action_templates}]
Available objects you can interact with: [{objects}]

Now it's your turn to take an action. Combine an action command with appropriate object(s) to form a valid action.
You should first reason step-by-step about the current situation. This reasoning process MUST be enclosed within <think> </think> tags.
Once you've finished your reasoning, you should choose a valid action for the current step and present it within <action> </action> tags.
\end{agentprompt}

\begin{agentprompt}{PromptFast}{ScienceWorld: Next-action prediction}
Your ScienceWorld task is: {task_description}
Prior to this step, you have already taken {step_count} step(s). Below are the most recent {history_length} observations and the corresponding actions you took: {action_history}
You are now at step {current_step} and your current observation is: {current_observation}
Available action commands: [{action_templates}]
Available objects you can interact with: [{objects}]

Now it's your turn to take an action. Combine an action command with appropriate object(s) to form a valid action.
You should choose a valid action for the current step and present it within <action> </action> tags.
\end{agentprompt}

\begin{agentprompt}{PromptReference}{ScienceWorld: Reference-conditioned inverse dynamics}
Your ScienceWorld task is: {task_description}
Prior to this step, you have already taken {step_count} step(s). Below are the most recent {history_length} observations and the corresponding actions you took: {action_history}
You are now at step {current_step} and your current observation is: {current_observation}

Complete the missing current action in this reference transition:
Current observation --[MISSING CURRENT ACTION]--> reference next observation.

The reference next observation is the environment state immediately after the missing current action. It may be a terminal state.
Reference next observation:
{reference_next_observation}

Available action commands: [{action_templates}]
Available objects you can interact with: [{objects}]

Now it's your turn to take one action for the current step. Predict the MISSING CURRENT ACTION that leads to the reference next observation.
Combine an action command with appropriate object(s), then present exactly one action within <action> </action> tags.
\end{agentprompt}

\subsection{History Management}
\label{app:memory-management}
We store the full interaction history $h_t$ as chronological
observation--action pairs.
During ActFirst-OPD rollout, these pairs are $(o_i,\tilde a_i)$.
Past reasoning $z_i$ and unexecuted actions $a_i$ from asynchronous
full responses are excluded.
The task instruction $q$ and current observation $o_t$ enter the
schematic input $\vc_t=q\oplus h_t\oplus o_t$ separately.

Each generation mode renders its own request from this interaction
record, with $o_{t+1}^{\mathrm{ref}}$ added only for inverse dynamics.
Each full response $y_t=(z_t,a_t)$ is generated asynchronously
from a snapshot of the interaction context $\vc_t$;
subsequent history updates do not change that request.
Different templates can retain different history suffixes because
their token lengths differ.

For each request, we begin with a copy of $h_t$ and count tokens
after applying the chat template.
If the prompt exceeds its budget, we repeatedly remove the oldest
complete observation--action pair and render the prompt again,
until it fits or no history remains.
The corresponding no-history template is used when the retained
history is empty.
This selection leaves the stored $h_t$ unchanged and does not
shorten the current observation, benchmark-specific admissible action
information, or reference next observation.
We use no learned summaries or semantic history compression.

\subsection{Thinking Budget}
\label{app:thinking}

We use budgeted decoding for full think-then-act responses
$y_t=(z_t,a_t)$ generated by $\pi_{\vtheta}$ from the interaction context $\vc_t$
during both training and evaluation.
The same procedure applies to asynchronous full-response generation
for distillation in ActFirst-OPD.
Following the Qwen3 thinking-budget approach,\footnote{\url{https://qwen.readthedocs.io/en/stable/getting_started/quickstart.html\#thinking-budget}}
we use at most two generation requests to leave room for action
generation when reasoning is long.
Fast actions $\tilde a_t$ follow a separate action-only decoding path.

If the first request finishes normally, we use its output directly.
If it reaches its token budget before action generation starts,
we insert an early-stop instruction adapted to the action format:

\begin{agentprompt}{PromptFast}{Inserted thinking-budget continuation}
Considering the limited time by the user, I have to give the action based on the thinking directly now.
</think>

<action>
\end{agentprompt}

If reasoning has already ended, the inserted text omits
\texttt{</think>}; if action generation has started, no text is inserted.
For an unfinished response with remaining budget, a second request
continues from $\vc_t$, the first-stage output, and any inserted tokens.
Across all three benchmarks, the first-stage and total response limits
are 384 and 512 tokens during training, and 1,920 and 2,048 during
evaluation.
The second request generates at most 128 tokens, subject to the remaining
total budget after counting both sampled and inserted tokens.
The first-stage limit bounds generated tokens rather than fixing
the length of $z_t$, since reasoning may end before this limit.

Actions are extracted from the \texttt{<action>} field.
In think-then-act interaction, responses marked as truncated are
treated as empty actions.
The budgeted decoder does not mark a response as truncated solely
because it reaches the length limit after producing \texttt{</action>}.
During ActFirst-OPD training, the asynchronously generated
full response $y_t$ is used for distillation, while its action
$a_t$ is not executed; the environment advances through
$\tilde a_t$.
At evaluation time, the student follows standard think-then-act
interaction and submits the extracted action $a_t$ to the environment.

For OPD training, manually inserted tokens remain in the shared
student--teacher context but are excluded from supervision and
loss normalization.
The corresponding masks and optimization details are given in
Appendix~\ref{app:opd-optimization}.

\subsection{OPD Optimization}
\label{app:opd-optimization}

Following prior multi-turn OPD studies~\citep{tcod}, we optimize a sampled policy surrogate of the reverse-KL objective
in Eq.~\ref{eq:opd-objective}.
An optimization batch $\mathcal{B}$ contains turn-level full
responses $y_b=(z_b,a_b)$ indexed by $b$, with token contexts
$\vc_{b,j}$ defined as in Section~\ref{sec:preliminaries}.
Let $\vtheta_{\mathrm{old},b}$ denote the student parameters used
to generate response $b$.
Turn-level full responses from completed rollout batches are stored in a first-in, first-out (FIFO) buffer.
Each update consumes responses generated by the current student version or the immediately preceding version, discarding older samples during batch construction.
Remaining responses stay buffered for subsequent updates, subject to the same version constraint; consumed responses are not reused.
The frozen teacher scores the generated tokens under the same contexts.

Let $m_{b,j}=1$ for a model-generated response token and $0$ for
manually inserted tokens or padding.
Both generated reasoning and action tokens are supervised, including
model-generated formatting tokens.
The inserted tokens described in Appendix~\ref{app:thinking} remain
in the student--teacher context but contribute neither to the loss
nor to its normalization.
At supervised positions, the token advantage is
\begin{equation}
    A_{b,j}
    =
    \operatorname{sg}\!\left[
        \log
        \frac{\pi_{\vphi}(x_{b,j}\mid\vc_{b,j})}
             {\pi_{\vtheta_{\mathrm{old},b}}(x_{b,j}\mid\vc_{b,j})}
    \right],
    \label{eq:opd-token-advantage}
\end{equation}
where $\operatorname{sg}$ denotes stop-gradient; masked positions
have zero advantage.
The distillation coefficient is set to one.
Token advantages are computed without environment rewards or advantage normalization.

Using stored sampling log-probabilities for the behavior-policy
denominator, we form the numerically bounded ratio
\begin{equation}
    \rho_{b,j}(\vtheta)
    =
    \exp\!\left[
        \operatorname{clip}\!\left(
            \log
            \frac{\pi_{\vtheta}(x_{b,j}\mid\vc_{b,j})}
                 {\pi_{\vtheta_{\mathrm{old},b}}(x_{b,j}\mid\vc_{b,j})},
            -20,20
        \right)
    \right].
    \label{eq:opd-ratio}
\end{equation}
The log-probability ratio is clipped to $[-20,20]$ for numerical
stability before exponentiation.
We use the proximal policy optimization (PPO) surrogate~\citep{ppo} with dual
clipping~\citep{dualclip}.
With clipping width $\delta=0.2$ and dual-clipping constant $\kappa=3$,
the minimized token loss is
\begin{align}
    \ell_{\mathrm{clip}}(\rho,A)
    &=
    -\min\!\left\{
        \rho A,\,
        \operatorname{clip}(\rho,1-\delta,1+\delta)A
    \right\},
    \nonumber\\
    \ell_{\mathrm{DC}}(\rho,A)
    &=
    \begin{cases}
        \min\!\left\{\ell_{\mathrm{clip}}(\rho,A),-\kappa A\right\},
            & A<0,\\
        \ell_{\mathrm{clip}}(\rho,A),
            & A\ge 0.
    \end{cases}
    \label{eq:opd-dual-clip}
\end{align}
Here, $\ell_{\mathrm{DC}}$ denotes the dual-clipped PPO token loss. 

We aggregate this loss using \texttt{seq-mean-token-mean}:
\begin{equation}
    \mathcal{L}_{\mathrm{sur}}(\vtheta)
    =
    \frac{1}{|\mathcal{B}|}
    \sum_{b\in\mathcal{B}}
    \frac{
        \sum_{j=1}^{L_b} m_{b,j}\,
        \ell_{\mathrm{DC}}\!\left(\rho_{b,j}(\vtheta),A_{b,j}\right)
    }{
        \sum_{j=1}^{L_b}m_{b,j}+\epsilon
    },
    \label{eq:opd-surrogate}
\end{equation}
where $L_b$ is the response length and $\epsilon=10^{-8}$.
Thus, losses are averaged over supervised tokens within each
full response, and the resulting response losses are averaged
across the batch.

All three generation modes share the student parameters $\vtheta$, which are updated by backpropagating through $\mathcal{L}_{\mathrm{sur}}$.
Sampling log-probabilities, teacher scores, and advantages remain
fixed during each update, and the teacher parameters $\vphi$ remain
frozen throughout training.
The main method adds no entropy penalty, separate reference-policy
KL loss, or auxiliary loss on the fast actions $\tilde a_t$.
%
%

\subsection{Transition Consistency Checks}
\label{app:transition-checks}

The checks in Section~\ref{sec:verification-fallback} compare the student's actual observations and available action information with their offline reference counterparts.
They use observable representations rather than simulator-state equality.
Retaining reference guidance requires a valid action, agreement of the environment's termination flag with the reference, and the benchmark-specific conditions below.
The generated action need not equal the reference action.

We normalize observation text before comparison by applying Unicode NFKC normalization and case folding, standardizing line endings, collapsing whitespace within each line, and removing empty lines.
The remaining line order is preserved except in ScienceWorld, as specified below.
Normalized observations are compared for exact equality, except for the WebShop search-result rule described below.

\paragraph{ALFWorld.}
An action is valid if its normalized text belongs to the current admissible-action set, excluding \texttt{help}.
We require the normalized next observation $o_{t+1}$ to equal $o_{t+1}^{\mathrm{ref}}$.
For nonterminal transitions, the next admissible-action sets must also match after text normalization and deduplication; this comparison is omitted for terminal transitions.

\paragraph{WebShop.}
A parsed \texttt{search[query]} action is valid when its query is nonempty and a search bar is available; a parsed \texttt{click[target]} action is valid when its lowercased target appears among the current clickable elements.
Observation comparison additionally removes click-notification lines of the form \texttt{You have clicked ...}.
We compare normalized next observations and, for nonterminal transitions, the available-action lists after case folding, whitespace normalization, and sorting.
The repeated instruction block is retained in these comparisons; its removal applies only to the reference observation shown in the inverse-dynamics prompt.
We permit different nonterminal search results when both the student and reference actions are searches and the product identifier in the reference's immediately following click action remains clickable.
The search-result mismatch is also allowed at the next turn, provided that the reference product remains clickable;
the subsequent click transition must satisfy the ordinary comparison rules.
Terminal transitions require matching normalized observations but no available-action comparison.

\paragraph{ScienceWorld.}
An action is valid if its normalized text appears in the current admissible-action list.
We compare normalized next observations and action-template sets, including for terminal transitions.
Observation normalization additionally sorts lines and object enumerations inside non-nested \texttt{(containing ...)} expressions and after the last \texttt{ is: } on each line.
Enumeration parsing ignores commas inside parentheses and retains clauses starting with \texttt{which} or \texttt{that} with their preceding item.
Action templates use the same text normalization, with duplicates removed and order ignored.
Possible-object aliases are excluded from matching because their enumeration can vary across resets.
The full admissible-action list is used for action validation, not set equality; separate room descriptions, inventory, reward, and score are not compared.

\medskip
\noindent
Across all three benchmarks, a failed check permanently disables reference conditioning for the remainder of the rollout, without resetting the environment.
An invalid reference-conditioned action proposal triggers one autonomous next-action prediction request from the unchanged interaction context before environment execution.
Reference actions are used internally for the stated checks but are not supplied to the student's prompt.

\section{Experiment Details}
\label{app:experiment-details}

\subsection{Benchmarks}
\label{app:benchmarks}

We use text observations and actions across all three benchmarks.
Table~\ref{tab:benchmark-details} summarizes environment versions,
task counts, and interaction limits.
Training and evaluation task instances are disjoint.

\begin{table}[t]
    \centering
    \caption{\textbf{Benchmark configurations.} Counts refer to task instances;
    the WebShop version denotes its base repository commit.}
    \label{tab:benchmark-details}
    \small
    \setlength{\tabcolsep}{5pt}
    \renewcommand{\arraystretch}{1.10}
    \begin{tabular*}{\linewidth}{@{\extracolsep{\fill}}lcccc@{}}
        \toprule
        Benchmark & Version & Training & Evaluation & Max. turns \\
        \midrule
        ALFWorld
        & 0.4.2 & 3,553 & 140 seen / 134 unseen & 30 \\
        WebShop
        & \texttt{64fa2a5} & 4,096 & 100 & 15 \\
        ScienceWorld
        & 1.2.3 & 2,294 & 150 & 30 \\
        \bottomrule
    \end{tabular*}
\end{table}

\paragraph{ALFWorld.}
ALFWorld~\citep{alfworld}\footnote{\url{https://github.com/alfworld/alfworld}}
contains six household task types,
reported in Table~\ref{tab:main-alfworld} as Pick (pick and place),
Look (examine an object under a light), Clean (clean and place),
Heat (heat and place), Cool (cool and place), and Pick2 (pick two
objects and place them).
The prompt lists the current admissible actions, excluding \texttt{help}, with object and receptacle arguments already specified, and instructs the student to select one.
We use the \texttt{json\_2.1.1} game files with TextWorld~\citep{textworld} 1.7.0,
training on \texttt{train} and evaluating separately on
\texttt{valid\_seen} and \texttt{valid\_unseen}.
The seen split uses rooms encountered during training with new
object configurations; the unseen split uses held-out rooms with
different layouts and receptacles.
Both splits cover the same six task types.

\paragraph{WebShop.}
WebShop~\citep{webshop}\footnote{\url{https://github.com/princeton-nlp/WebShop}}
requires agents to search for and purchase
products satisfying language instructions.
The prompt lists the current \texttt{click[target]} actions and, when search is available, a \texttt{search[query]} template whose query is generated by the student.
Following prior multi-turn OPD studies~\citep{tcod}, we use its text interface with the full product catalog and a custom
split: goal IDs 0--4095 for training and 4096--4195 for evaluation.
The goal-construction seed is fixed at 233.

\paragraph{ScienceWorld.}
ScienceWorld~\citep{scienceworld}\footnote{\url{https://github.com/allenai/ScienceWorld}}
provides interactive science tasks
covering physical, chemical, and biological processes.
To limit prompt length, we provide action templates and possible object referents separately for the student to construct commands.
The environment's admissible-action list is used for action validation.
We use the \texttt{easy} environment preset and a custom split:
training uses the first half of the variations from 17 task types,
while evaluation uses the last five variations from each of all
30 task types.
The evaluation set therefore contains 85 instances from task types
covered during training and 65 from the remaining 13 types.

\paragraph{Evaluation Metrics.}
For $N$ evaluation tasks, let $u_i\in\{0,1\}$ indicate task success,
$\mathrm{score}_i\in[0,100]$ denote the task score, and $T_i$ denote the number
of interaction turns used.
We report
\begin{equation}
    \mathrm{SR}=\frac{100}{N}\sum_{i=1}^{N}u_i,
    \qquad
    \mathrm{Score}=\frac{1}{N}\sum_{i=1}^{N}\mathrm{score}_i,
    \qquad
    \mathrm{Round}
=
\overline{T}=\frac{1}{N}\sum_{i=1}^{N}T_i.
    \label{eq:evaluation-metrics}
\end{equation}
SR is expressed as a percentage; Score is reported only for
WebShop and ScienceWorld to capture partial task completion.

ALFWorld success requires satisfying the task goal within the turn
limit; task-type and overall results pool the seen and unseen tasks.
WebShop assigns partial credit for matching the requested product
attributes, options, and price, weighted by product-type match.
Its task score is the terminal purchase reward scaled to $[0,100]$,
with zero assigned if no purchase is made.
ScienceWorld assigns partial credit for completing task subgoals;
we use the highest environment score reached during the episode,
clipped to $[0,100]$.
Success on WebShop and ScienceWorld requires a full score,
implemented as $\mathrm{score}_i\geq99.999$.

Round averages interaction turns over both successful and failed episodes, including invalid action attempts.
On ALFWorld, episodes end upon task success or at the 30-turn limit, so failed episodes contribute 30 turns.
On WebShop and ScienceWorld, unsuccessful purchases or terminal failures can end an episode early; fewer turns therefore do not necessarily indicate better task performance.

\subsection{Teacher Training and Reference Construction}
\label{app:teacher-reference}

\paragraph{Teacher Training.}
Each benchmark uses a task-specialized Qwen3-8B teacher, shared
across student sizes and frozen throughout OPD.
The ALFWorld teacher is trained by SFT on successful trajectories
from the GiGPO-Qwen2.5-7B model, followed by GiGPO~\citep{gigpo}.
The WebShop teacher uses score-filtered SFT on trajectories from the
corresponding GiGPO-Qwen2.5-7B model, followed by GiGPO~\citep{gigpo}.
The ScienceWorld teacher is trained directly with GiGPO.

\paragraph{Reference Trajectories.}
To obtain high-quality references, we select trajectories based on task success or score, validate them through environment
replay, and favor shorter action sequences.
For ALFWorld, candidates include public demonstrations from Hugging Face, hand-coded solutions, and ten rollouts per training
task from each of GiGPO-Qwen2.5-7B and our Qwen3-8B teacher.
We select the shortest candidate marked successful and validate it.
For WebShop, we combine public demonstrations from Hugging Face, ten teacher rollouts per training task, and oracle solutions
constructed from training-task goals.
Candidates are matched to the corresponding training-task instructions and replayed;
selection prioritizes higher score and then fewer actions.
For ScienceWorld, we combine ten teacher rollouts per training task, simulator-generated gold paths, and task-specific scripted
solutions.
Candidates are shortened while preserving successful replay.
We replay selected action sequences offline under the training environment configuration and cache the initial and subsequent
observations.
The cached next observations serve as local transition targets
$o_{t+1}^{\mathrm{ref}}$ for reference-conditioned inverse dynamics.
The resulting caches cover all training tasks across the three benchmarks (100\% coverage), with one selected trajectory per task.

\paragraph{Offline Cost.}
Teacher training and reference construction are completed before
student training and excluded from the reported training wall-clock time.
For a fixed task set and environment, the resulting teachers and
reference caches can be reused across student sizes and repeated
training runs, amortizing their one-time preparation costs.
Table~\ref{tab:offline-cost} reports teacher training and candidate
rollout collection separately.

\begin{table}[t]
    \centering
    \caption{\textbf{Offline preparation costs.}
    Times are in hours, with eight GPUs per job.
    Teacher training time includes all listed stages.
    Ten candidates are generated per task for each source model.
   }
    \label{tab:offline-cost}
    \small
    \setlength{\tabcolsep}{5pt}
    \renewcommand{\arraystretch}{1.12}
    \begin{tabular*}{\linewidth}
        {@{\extracolsep{\fill}}lrrrrr@{}}
        \toprule
        & \multicolumn{3}{c}{Teacher training}
        & \multicolumn{2}{c}{Candidate collection} \\
        \cmidrule(lr){2-4}
        \cmidrule(l){5-6}
        Benchmark
        & Time (h) & SFT steps & GiGPO steps
        & Time (h) & Candidates \\
        \midrule
        ALFWorld
        & 4.18  & 250 & 250
        & 8.50  & 71,060 \\
        WebShop
        & 22.05 & 250 & 2,200
        & 4.81  & 40,960 \\
        ScienceWorld
        & 18.27 & \textemdash & 144
        & 19.88 & 22,940 \\
        \bottomrule
    \end{tabular*}
\end{table}

\subsection{Training Protocols and Hyperparameters}
\label{app:training-hyperparameters}

Table~\ref{tab:training-hyperparameters} summarizes student training configurations
across benchmarks.
Within each benchmark--model setting, methods share the initial
Qwen3 checkpoint and frozen teacher.
In the training pipeline, rollout collection and student
optimization proceed asynchronously.
The OPD loss, token masks, and sample staleness control are detailed
in Appendix~\ref{app:opd-optimization}.

\begin{table}[t]
    \centering
    \caption{\textbf{Training hyperparameters across all benchmarks.}}
    \label{tab:training-hyperparameters}
    \small
    \setlength{\tabcolsep}{4pt}
    \renewcommand{\arraystretch}{1.02}
    \begin{tabular*}{\linewidth}{
        @{\hspace{0.8em}}l@{\extracolsep{\fill}}l@{}
    }
        \toprule
        \multicolumn{1}{@{}l}{\textbf{Hyperparameter}}
        & \textbf{Value} \\
        \midrule

        \multicolumn{2}{@{}l}{
            \textbf{Algorithm and Optimization}
        } \\
        Distillation objective & Student-to-teacher reverse KL \\
        PPO clipping & 0.2 \\
        Dual-clipping constant & 3.0 \\
        Optimizer & AdamW~\citep{adamw} \\
        Learning rate & $10^{-6}$ \\
        Learning-rate schedule & Constant, no warmup \\
        Adam coefficients $(\beta_1,\beta_2)$ & $(0.9,0.999)$ \\
        Weight decay & 0.01 \\
        Gradient norm clipping & 1.0 \\
        Sample staleness control & Version lag $\leq 1$ \\
        Loss aggregation & \texttt{seq-mean-token-mean} \\
        \midrule

        \multicolumn{2}{@{}l}{\textbf{Training}} \\
        Training updates & 250 \\
        Rollout batch & 16 tasks \\
        Rollout trajectories per task & 1 \\
        Optimization batch & 64 turn-level full responses \\
        Optimization epochs per batch & 1 \\
        Weight synchronization interval & Every update \\
        \midrule

        \multicolumn{2}{@{}l}{\textbf{Models and Generation}} \\
        Student models & Qwen3-0.6B, 1.7B, 4B \\
        Teacher model & Qwen3-8B \\
        Prompt limit & 10,240 tokens \\
        Full-response limit & 512 tokens \\
        Thinking budget (first stage) & 384 tokens \\
        Action continuation budget & Up to 128 tokens \\
        Sampling temperature & 1.0 \\
        Top-$p$ & 1.0 \\
        Top-$k$ filtering & Disabled \\
        Thinking mode & Enabled (\texttt{enable\_thinking=True}) \\
        Rollout engine seed & 42 \\
        \midrule

        \multicolumn{2}{@{}l}{\textbf{Distributed Training}} \\
        Hardware & 8 A100-SXM4-80GB GPUs \\
        Rollout GPUs & 4 \\
        Teacher-scoring GPUs & 2 \\
        Student parameter optimization GPUs & 2 \\
        Rollout engines per GPU & 1 \\
        Runners per rollout engine & 4 \\
        Total rollout runners & 16 \\
        Inference tensor parallelism & 1 \\
        Training sequence parallelism & 2 (Ulysses~\citep{ulysses}) \\
        Precision & \texttt{bfloat16} \\
        Dynamic batching token budget & 16,384 tokens per GPU \\
        vLLM GPU memory fraction & 0.7 \\
        Prefix caching & Disabled \\
        \midrule

        \multicolumn{2}{@{}l}{
            \textbf{Environment-Specific Turn Limits}
        } \\
        ALFWorld & 30 \\
        WebShop & 15 \\
        ScienceWorld & 30 \\
        \bottomrule
    \end{tabular*}
\end{table}

\paragraph{Training Time.}
Training wall-clock time in Tables~\ref{tab:main-alfworld}
and~\ref{tab:main-webshop-scienceworld} spans training-process
launch to exit, including initialization, rollout collection,
teacher scoring, student updates, checkpoint saving, and shutdown.
Offline teacher training, reference construction, and subsequent
evaluation are excluded.

\subsection{Evaluation Protocols and Hyperparameters}
\label{app:evaluation-details}

\begin{table}[t]
    \centering
    \caption{\textbf{Evaluation hyperparameters across all benchmarks.}}
    \label{tab:evaluation-hyperparameters}
    \small
    \setlength{\tabcolsep}{4pt}
    \renewcommand{\arraystretch}{1.02}
    \begin{tabular*}{\linewidth}{
        @{\hspace{0.8em}}l@{\extracolsep{\fill}}l@{}
    }
        \toprule
        \multicolumn{1}{@{}l}{\textbf{Hyperparameter}}
        & \textbf{Value} \\
        \midrule

        \multicolumn{2}{@{}l}{\textbf{Evaluation Protocol}} \\
        Interaction mode & Think-then-act \\
        Reference information & None \\
        Evaluation seeds & 42, 43, 44 \\
        Trajectories per task and seed & 1 \\
        Evaluation batch & 64 tasks \\
        \midrule

        \multicolumn{2}{@{}l}{\textbf{Generation}} \\
        Prompt limit & 20,480 tokens \\
        Full-response limit & 2,048 tokens \\
        Thinking budget (first stage) & 1,920 tokens \\
        Action continuation budget & Up to 128 tokens \\
        Sampling temperature & 0.4 \\
        Top-$p$ & 1.0 \\
        Top-$k$ filtering & Disabled \\
        Thinking mode & Enabled (\texttt{enable\_thinking=True}) \\
        \midrule

        \multicolumn{2}{@{}l}{\textbf{Distributed Inference}} \\
        Hardware & 8 A100-SXM4-80GB GPUs \\
        Inference engines per GPU & 1 \\
        Inference tensor parallelism & 1 \\
        Precision & \texttt{bfloat16} \\
        vLLM GPU memory fraction & 0.82 \\
        Prefix caching & Disabled \\
        Chunked prefill & Enabled \\
        \midrule

        \multicolumn{2}{@{}l}{
            \textbf{Benchmark-Specific Runners per Engine}
        } \\
        ALFWorld & 8 \\
        WebShop & 4 \\
        ScienceWorld & 4 \\
        \midrule

        \multicolumn{2}{@{}l}{
            \textbf{Environment-Specific Turn Limits}
        } \\
        ALFWorld & 30 \\
        WebShop & 15 \\
        ScienceWorld & 30 \\
        \bottomrule
    \end{tabular*}
\end{table}

All methods use standard think-then-act interaction without
reference information, with the settings in
Table~\ref{tab:evaluation-hyperparameters}.
Task splits and metric definitions are provided in
Appendix~\ref{app:benchmarks}; budgeted decoding and action
extraction follow Appendix~\ref{app:thinking}.

For a fixed checkpoint, let $M_r$ denote a task-performance
metric evaluated with seed $r\in\mathcal{S}=\{42,43,44\}$.
We report the mean and sample standard deviation
(Tables~\ref{tab:main-alfworld}
and~\ref{tab:main-webshop-scienceworld};
Figure~\ref{fig:alfworld-performance-time}):
\begin{equation}
    \overline{M}
    = \frac{1}{3}\sum_{r\in\mathcal{S}} M_r,
    \qquad
    \operatorname{SD}(M)
    = \sqrt{
        \frac{1}{3-1}
        \sum_{r\in\mathcal{S}}
        \left(M_r-\overline{M}\right)^2
    }.
    \label{eq:evaluation-aggregation}
\end{equation}

\subsection{Rollout Efficiency and Quality Measurements}
\label{app:analysis-protocols}

\paragraph{Rollout Efficiency.}
For the rollout efficiency experiment in
Figure~\ref{fig:efficiency-analysis}(a,b), we use the initial
Qwen3-1.7B model and 1,024 initial ALFWorld contexts, sampled
proportionally across the six training task types with seed 42.
Full responses use the standard think-then-act prompt; fast actions
use the action-only prompt without reference information.
We vary the cap $b_{\max}$ on concurrent generation requests,
shared by fast-action and full-response requests,
and the virtual interaction horizon $T$
defined in Section~\ref{sec:speedup-analysis}.
Table~\ref{tab:scheduling-hyperparameters} summarizes the settings.

\begin{table}[t]
    \centering
    \caption{\textbf{Rollout efficiency experiment settings.}}
    \label{tab:scheduling-hyperparameters}
    \small
    \setlength{\tabcolsep}{4pt}
    \renewcommand{\arraystretch}{1.02}
    \begin{tabular*}{\linewidth}{
        @{\hspace{0.8em}}l@{\extracolsep{\fill}}l@{}
    }
        \toprule
        \multicolumn{1}{@{}l}{\textbf{Parameter}}
        & \textbf{Value} \\
        \midrule

        \multicolumn{2}{@{}l}{\textbf{Workload}} \\
        Tasks per batch & 16 \\
        Fast-action output length & 16 tokens \\
        Full-response output length & 128 or 512 tokens \\
        Sampling temperature & 1.0 \\
        Top-$p$ & 1.0 \\
        Top-$k$ filtering & Disabled \\
        \midrule

        \multicolumn{2}{@{}l}{\textbf{Inference}} \\
        Hardware per worker & 1 A100-SXM4-80GB GPU \\
        Tensor parallelism & 1 \\
        Precision & \texttt{bfloat16} \\
        vLLM GPU memory fraction & 0.85 \\
        Maximum tokens per scheduling step & 16,384 \\
        Prefix caching & Disabled \\
        Chunked prefill & Enabled \\
        CUDA graphs & Enabled \\
        \midrule

        \multicolumn{2}{@{}l}{\textbf{Scans and Repeats}} \\
        Concurrency scan: $b_{\max}$
        & $\{16,32,64,128,256\}$ \\
        Fixed horizon for concurrency scan & $T=30$ \\
        Horizon scan: $T$ & $\{1,2,4,8,16,30\}$ \\
        Fixed concurrency for horizon scan & $b_{\max}=256$ \\
        Paired timing repeats & 3 \\
        \bottomrule
    \end{tabular*}
\end{table}

Under Vanilla OPD scheduling, each task's next turn waits for
its full response.
Under ActFirst-OPD scheduling, fast-action and full-response
requests are submitted together, and fast-action completion
releases the next turn.
Tasks run concurrently within each batch, which ends only after
all requests finish.
Timing includes queueing, prefill, decoding, and request scheduling,
but excludes initialization, warmup, input tokenization, and output
detokenization.
Each paired repeat uses the same engine, contexts, and full-response
sampling seeds, with strategy order alternated across repeats.

Let $B$ denote a batch of 16 task indices, and let
$C_{\mathrm{Vanilla},B}^{(k)}$ and
$C_{\mathrm{ActFirst},B}^{(k)}$ denote its measured completion
times in paired repeat $k$.
The rollout speedup in each repeat is
\begin{equation}
    S_{\mathrm{real}}^{(k)}
    =
    \frac{
        \sum_B C_{\mathrm{Vanilla},B}^{(k)}
    }{
        \sum_B C_{\mathrm{ActFirst},B}^{(k)}
    },
    \qquad k\in\{1,2,3\},
    \label{eq:measured-rollout-speedup}
\end{equation}
where the sums cover all 64 batches.
We report the mean and sample standard deviation of these
three ratios.

Idealized speedup estimates use independent single-request
measurements, including prefill and decoding.
For the context of task $i$, let
$\widehat\ell_{\mathrm{fast},i}$ and
$\widehat\ell_{\mathrm{full},i}$ denote the median fast-action
and full-response latencies over three measurements.
Treating these calibrated latencies as constant across $T$
virtual turns, the estimated batch completion times under
unrestricted overlap are
\begin{equation}
    \begin{aligned}
        \widehat C_{\mathrm{Vanilla},B}
        &=
        \max_{i\in B}
        T\widehat\ell_{\mathrm{full},i}, \\
        \widehat C_{\mathrm{ActFirst},B}
        &=
        \max_{i\in B}
        \left[
            (T-1)\widehat\ell_{\mathrm{fast},i}
            +
            \max\left\{
                \widehat\ell_{\mathrm{fast},i},
                \widehat\ell_{\mathrm{full},i}
            \right\}
        \right], \\
        \widehat S_{\mathrm{ideal}}(T)
        &=
        \frac{
            \sum_B \widehat C_{\mathrm{Vanilla},B}
        }{
            \sum_B \widehat C_{\mathrm{ActFirst},B}
        }.
    \end{aligned}
    \label{eq:calibrated-rollout-speedup}
\end{equation}
The maximum over tasks reflects that a batch waits for its
slowest task.
The inner maximum for ActFirst-OPD accounts for both requests
at the final turn.
When $\widehat\ell_{\mathrm{fast},i}
\leq\widehat\ell_{\mathrm{full},i}$, its per-task estimate
reduces to the expression in
Proposition~\ref{prop:rollout-speedup}.
These calibrated estimates extend the idealized analysis
to task batches and are not hardware upper bounds.

\paragraph{Environment-Transition Throughput.}
For Qwen3-1.7B on ALFWorld, both methods use 16 tasks per rollout
batch, four rollout engines, and four runners per engine.
Let $N_{\mathrm{env}}$ count completed environment transitions
within a measurement window of $\Delta t$ seconds.
Throughput is
\begin{equation}
    \mathrm{Throughput}
    =
    \frac{N_{\mathrm{env}}}{\Delta t},
    \label{eq:environment-throughput}
\end{equation}
measured in environment transitions per second
(\mbox{env.\ steps / s}).

The full-training window starts at the first training-sample
request and ends when update 250 finishes; the early window covers
the first 600 seconds from the same origin.
Both include rollout, scoring, optimization, and waiting time,
but exclude initialization and shutdown.
Only transitions completed within each window are counted.
Table~\ref{tab:throughput-measurements} reports the measurements.

\begin{table}[t]
    \centering
    \caption{\textbf{Environment-transition throughput on ALFWorld
    with Qwen3-1.7B.}
    Measurements use training runs separate from those reported in
    Table~\ref{tab:main-alfworld}.
    For each method, both windows use the same run and start
    when rollout collection begins.
    The full-training window ends when update 250 completes;
    the early window covers the first 600 seconds.
    Initialization and final shutdown are excluded.
    }
    \label{tab:throughput-measurements}
    \small
    \setlength{\tabcolsep}{4pt}
    \renewcommand{\arraystretch}{1.02}
    \begin{tabular*}{\linewidth}{
        @{\extracolsep{\fill}}llrrr@{}
    }
        \toprule
        Window & Method & Transitions & Window time (s) & Env. steps / s \\
        \midrule
        Full training & Vanilla OPD & 55,525 & 15,202 & 3.65 \\
                      & ActFirst-OPD & 38,616 & 6,796 & 5.68 \\
        \midrule
        First 10 min & Vanilla OPD & 1,289 & 600 & 2.15 \\
                     & ActFirst-OPD & 3,803 & 600 & 6.34 \\
        \bottomrule
    \end{tabular*}
\end{table}

\paragraph{Rollout Quality.}
We compare think-then-act, direct action without reference next observations,
and ActFirst-OPD's reference-conditioned interaction with
autonomous fallback.
At each of updates 1, 125, and 250, all strategies use the same
frozen Qwen3-1.7B Vanilla OPD checkpoint and generate one trajectory
per ALFWorld training task (3,553 tasks).
Checkpoints are analyzed separately.
All strategies use temperature 1.0, top-$p=1.0$, no top-$k$
filtering, and a 30-turn limit, with sampling seeds matched
across strategies for each task and turn.
Think-then-act uses the training thinking budget in
Appendix~\ref{app:thinking}; all action-only requests stop at
\texttt{</action>} or 512 generated tokens.
ActFirst-OPD rollouts retain the training transition checks
and persistent autonomous fallback, but omit asynchronous
full-response generation, teacher scoring, and parameter updates.

An executed action is valid if parsing succeeds and the action,
after stripping surrounding whitespace, exactly matches an entry
in the admissible-action list.
Repetition is determined from environment state--action pairs.
For this offline measurement, state equality requires identical
sets of positive environment facts and the same task-success flag.
Actions are normalized for pair comparison.

A valid action counts as an unproductive repetition if the same state--action
pair was previously executed validly, no new normalized observation
has appeared since the end of that pair's most recent valid
execution, and the current action also produces no new observation.
An observation is new if its normalized text has not appeared
earlier in the rollout.
Text normalization uses Unicode NFKC, case folding, and whitespace
normalization.
The observation record includes $o_1$ and feedback from invalid
attempts, is maintained independently of prompt-history truncation,
and excludes reference observations and model outputs.

For each checkpoint--strategy pair, let $r_i\in\{0,1\}$ indicate
whether task $i$ contains at least one repetition as defined above,
$u_i\in\{0,1\}$ indicate success (Appendix~\ref{app:benchmarks}), and $T_i$ denote its executed
interaction turns (Appendix~\ref{app:benchmarks}).
For $N=3{,}553$ tasks, the metrics in
Figure~\ref{fig:rollout-quality}(a--c) are computed as
\begin{equation}
    \begin{alignedat}{2}
        &\text{(a) Tasks with }\geq 1\text{ repeated action:}\quad
        &\mathrm{Repeat}\;(\%)
        &= \frac{100}{N}\sum_{i=1}^{N}r_i, \\[3pt]
        &\text{(b) Successful trajectory yield:}\quad
        &\mathrm{Yield}
        &= \frac{100\sum_{i=1}^{N}u_i}
                 {\sum_{i=1}^{N}T_i}, \\[3pt]
        &\text{(c) Average rollout turns:}\quad
        &\overline{T}
        &= \frac{1}{N}\sum_{i=1}^{N}T_i.
    \end{alignedat}
    \label{eq:rollout-quality-metrics}
\end{equation}
Repeat is the percentage of tasks containing at least one
repetition; Yield counts successful trajectories per
100 environment transitions; $\overline{T}$ is the mean number
of executed turns across all tasks.
Turn counts include invalid attempts and failed trajectories.
All failed trajectories in this collection reach 30 turns,
so $\overline{T}$ also equals completion cost with a 30-turn
failure penalty.

\section{Extended Related Work}
\label{app:extended-related-work}

We extend Section~\ref{sec:related_work} by comparing how related OPD and agent distillation methods collect or reuse trajectories, provide teacher supervision, and improve training efficiency.
ReOPD~\citep{reopd} samples a single-turn student response conditioned on each selected teacher-trajectory prefix and trains the student with teacher supervision, without online environment interaction.
Our Reference-Prefix Replay (100\%) ablation (Table~\ref{tab:ablation_average_across_scales}) similarly uses reference trajectories to construct student training contexts, but executes the reference actions in the environment. 
This ablation evaluates the effect of replacing student-generated actions with reference actions within our framework.
It does not reproduce ReOPD's off-environment training or prefix-sampling schedule.
Guided-OPD~\citep{guidedopd} mixes teacher- and student-generated turns under a curriculum that gradually removes teacher intervention.
SAGE-OPD~\citep{sageopd} combines turn-level intervention based on environment feedback and teacher judgments with teacher-confidence weighting of the distillation loss.
Structured agent distillation (SAD)~\citep{sad} segments ReAct trajectories into reasoning and action spans for separate supervision.
Multi-rollout on-policy distillation (MOPD)~\citep{mopd} conditions the teacher on successful and failed peer rollouts to improve
token-level supervision.
Near-policy distillation~\citep{npd} decouples student-response generation from optimization, using precomputed teacher logits and sequence packing for efficient training and sample filtering to address policy lag and sample noise.
Lightning OPD~\citep{lightningopd} precomputes teacher log-probabilities on rollouts from an SFT-initialized student and reuses them during distillation, eliminating live teacher serving during optimization.
These works study supervision quality and distillation efficiency from different perspectives.
Our work focuses on the efficiency of online experience collection in multi-turn agent OPD, specifically reducing reasoning-induced interaction delays while maintaining rollout quality.
Reference next observations guide fast student-generated actions, with autonomous next-action prediction after divergence, while the student asynchronously generates a full response from each interaction context for teacher supervision.
This retains online environment interaction and full-response distillation while removing reasoning from the critical path of environment transitions.

\section{Additional Discussion}
\label{app:additional-discussion}

\paragraph{Scope of On-Policy Distillation.} In ActFirst-OPD, we distinguish two aspects of on-policy distillation: the distribution of rollout contexts in the full interaction trajectory, and the on-policy sampling of turn-level full responses used for distillation, conditional on those contexts.
The environment advances through fast actions $\tilde a_t$, using reference-conditioned inverse dynamics before divergence and autonomous next-action prediction afterward.
At each collected interaction context $\vc_t$, the student generates a full response $y_t=(z_t,a_t)$ without conditioning on the reference next observation.
The frozen teacher supervises its tokens, while $a_t$ is not executed.
Thus, turn-level full responses are sampled on-policy conditional on the collected contexts, while their context distribution
need not match that induced by the student's think-then-act policy.

\paragraph{Context Distribution and Exposure Bias.} Early OPD methods, including MiniLLM~\citep{minillm} and GKD~\citep{gkd}, address train--inference mismatch in response prefixes by incorporating student-generated sequences into distillation.
This mismatch concerns prefixes within a response, conditional on a given input.
In our setting, these are token prefixes $\vx_{t,<j}$ within each full response $y_t$, distinct from the distribution of interaction contexts $\vc_t$ across interaction turns.
Reference-conditioned inverse dynamics may improve task progress, as suggested by the rollout-quality results in Figure~\ref{fig:rollout-quality}(b,c).
At the same time, reference guidance may reduce exposure to states that could arise from accumulated errors in think-then-act interaction.
Autonomous next-action prediction continues interaction from actual deviations, but remains action-only and need not induce the same context distribution as think-then-act interaction at inference time.
The results in Tables~\ref{tab:main-alfworld} and~\ref{tab:main-webshop-scienceworld} support the practical value of this design in the evaluated settings, without establishing that exposure bias at inference time is eliminated or uniformly reduced.

\section{Additional Experimental Results}
\label{app:additional-experimental-results}

\paragraph{Performance--Time Trade-off.}
Figure~\ref{fig:alfworld-performance-time} plots mean evaluation SR against cumulative
training time on ALFWorld for three Qwen3 student sizes.
ActFirst-OPD completes 250 updates within two hours at all three sizes.
At approximately 1.5 hours, the 0.6B, 1.7B, and 4B students
reach SRs of 58.39\%, 69.34\%, and 80.66\%, respectively,
compared with 5.72\%, 25.43\%, and 42.94\% under Vanilla OPD.
These comparisons show higher SR under comparable training-time
budgets, complementing the fixed-update results in
Table~\ref{tab:main-alfworld}.

\begin{figure}[t]
    \centering
    \includegraphics[width=\linewidth]{
        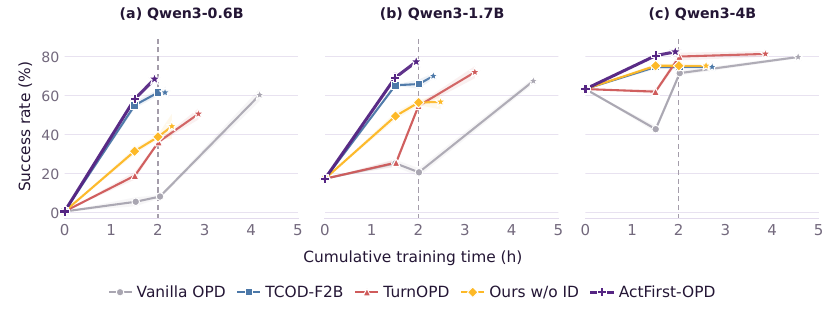
    }
    \caption{\textbf{Performance--time trade-off on ALFWorld for
    Qwen3-0.6B, 1.7B, and 4B students.}
    SR aggregates seen and unseen tasks; shading indicates $\pm1$ standard
    deviation across three evaluation seeds.
    Stars mark the final checkpoints in Table~\ref{tab:main-alfworld};
    the dashed line marks two hours of training.}
    \label{fig:alfworld-performance-time}
\end{figure}

\begin{table}[t]
    \centering
    \caption{\textbf{Cumulative reference alignment over 250 training updates.}
    Cumulative rates are computed by aggregating the corresponding counts across complete rollouts in rollout batches finished by update 250.
    First-action and full-trajectory alignment use the number of rollout tasks with reference guidance as the denominator;
    turn coverage uses all executed turns, including fallback turns.}
    \label{tab:cumulative-reference-alignment}
    \small
    \setlength{\tabcolsep}{4pt}
    \renewcommand{\arraystretch}{1.08}
    \begin{tabular*}{\linewidth}{@{\extracolsep{\fill}}lcccc@{}}
        \toprule
        Benchmark
        & Student
        & \begin{tabular}[c]{@{}c@{}}
            First-action\\alignment (\%)
          \end{tabular}
        & \begin{tabular}[c]{@{}c@{}}
            Full-trajectory\\alignment (\%)
          \end{tabular}
        & \begin{tabular}[c]{@{}c@{}}
            Aligned-turn\\coverage (\%)
          \end{tabular} \\
        \midrule
        \multirow{3}{*}{ALFWorld}
        & 0.6B & 84.86 & 32.36 & 18.19 \\
        & 1.7B & 89.38 & 49.75 & 33.59 \\
        & 4B   & 93.39 & 85.63 & 64.05 \\
        \midrule
        \multirow{3}{*}{WebShop}
        & 0.6B & 49.33 & 11.80 & 18.77 \\
        & 1.7B & 56.89 & 21.24 & 23.50 \\
        & 4B   & 60.85 & 49.87 & 51.77 \\
        \midrule
        \multirow{3}{*}{ScienceWorld}
        & 0.6B & 72.56 & 27.74 & 10.44 \\
        & 1.7B & 62.57 & 40.44 & 14.99 \\
        & 4B   & 94.67 & 67.87 & 33.27 \\
        \bottomrule
    \end{tabular*}
\end{table}

\paragraph{Reference Alignment During Training.}
We quantify reference alignment before fallback at the first-action,
trajectory, and turn levels.
Figures~\ref{fig:reference-alignment-first-action}--%
\ref{fig:reference-alignment-turn-coverage}
report rates computed from counts aggregated over the latest 10 completed rollout batches
at each update, using all available batches when fewer than 10
have completed.
Endpoint labels show the final-window rates at update 250;
Table~\ref{tab:cumulative-reference-alignment} reports
full-training cumulative rates.

Larger students achieve higher cumulative full-trajectory alignment
and aligned-turn coverage within each benchmark
(Table~\ref{tab:cumulative-reference-alignment}).
First-action alignment consistently exceeds full-trajectory
alignment, indicating that a matching initial transition does
not guarantee complete reference following.
Aligned-turn coverage also captures reference-aligned prefixes
in rollouts that later diverge.
Windowed alignment rates do not increase monotonically during
training
(Figures~\ref{fig:reference-alignment-first-action}--%
\ref{fig:reference-alignment-turn-coverage}).

\begin{figure}[t]
    \centering
    \includegraphics[
        width=\linewidth,
        height=0.42\textheight,
        keepaspectratio
    ]{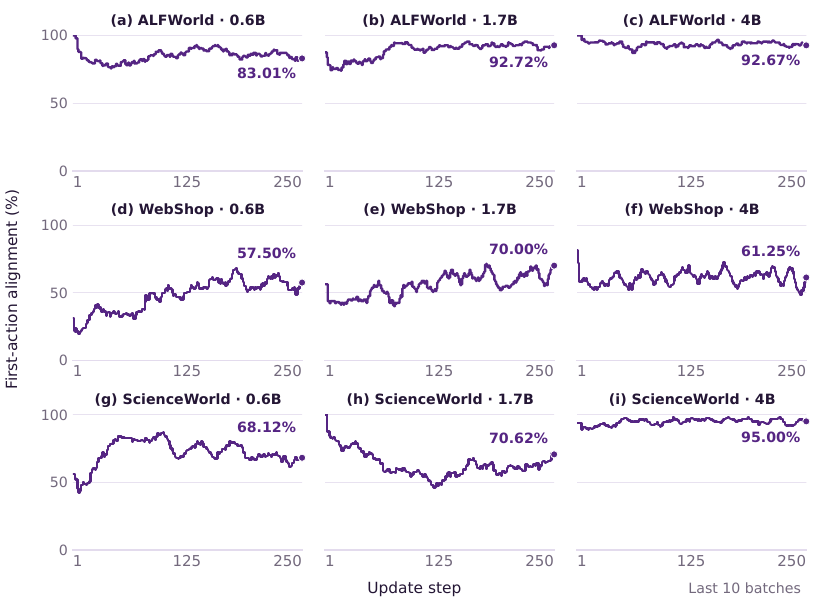}
    \caption{\textbf{First-action reference alignment during training.}
    Fraction of rollout tasks with reference guidance whose first
    transition passes the transition consistency check.}
    \label{fig:reference-alignment-first-action}
\end{figure}

\begin{figure}[t]
    \centering
    \includegraphics[
        width=\linewidth,
        height=0.40\textheight,
        keepaspectratio
    ]{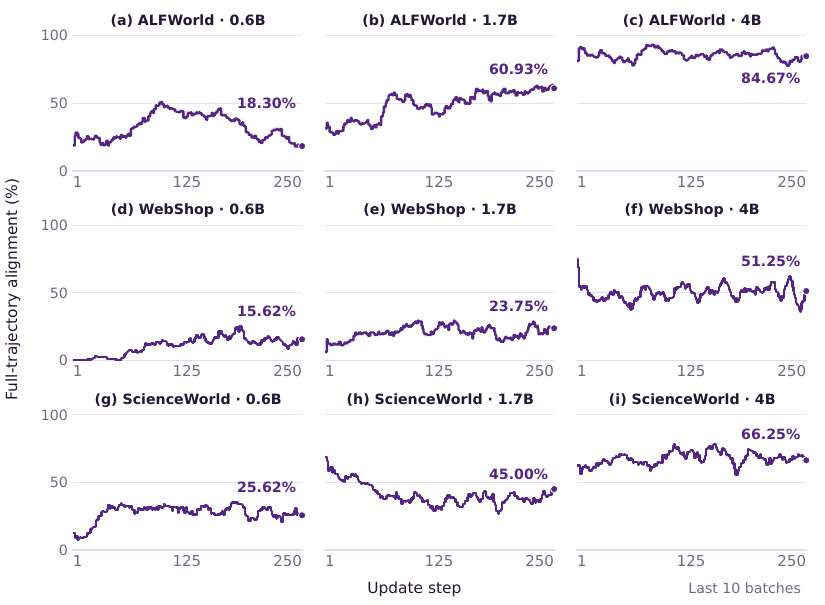}
    \caption{\textbf{Full-trajectory reference alignment during training.}
    Fraction of rollouts with reference guidance that complete
    the reference trajectory and terminate without triggering fallback.}
    \label{fig:reference-alignment-full-trajectory}

    \vspace{0.5em}

    \includegraphics[
        width=\linewidth,
        height=0.40\textheight,
        keepaspectratio
    ]{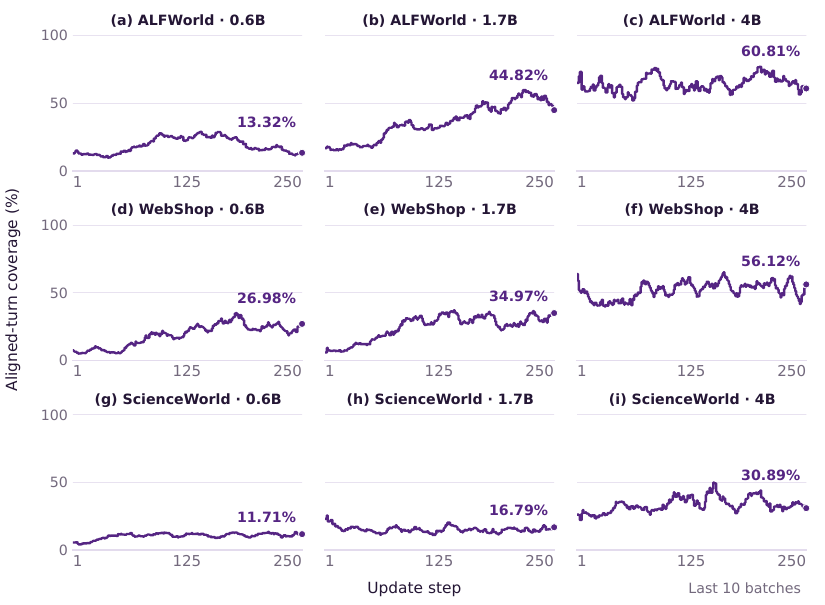}
    \caption{\textbf{Reference-aligned turn coverage during training.}
    Reference-aligned turns as a fraction of all executed turns,
    including fallback turns.}
    \label{fig:reference-alignment-turn-coverage}
\end{figure}





\end{document}